\documentclass{article}

\usepackage{microtype}
\usepackage{graphicx}
\usepackage{booktabs} % for professional tables

\usepackage{hyperref}

\usepackage[accepted]{icml2023}

\usepackage{algorithm}
\usepackage{algorithmic}

\usepackage{dirtytalk}
\usepackage{url}
\usepackage{lipsum}
\usepackage{algorithm}

\usepackage{booktabs} % for professional tables
\usepackage{multirow}
\usepackage{comment}
\usepackage[utf8]{inputenc}
\usepackage{graphicx}
\usepackage{xcolor}         % colors
\usepackage{amsmath}
\usepackage{amssymb}
\usepackage{amsthm}
\usepackage{bm}
\usepackage{bbm}
\usepackage{nicefrac}
\usepackage{mathtools}
\usepackage{tablefootnote}

\newcommand\myeq[1]{\stackrel{\mathclap{\normalfont\tiny\mbox{#1}}}{=}}

\newcommand\mygeq[1]{\stackrel{\mathclap{\normalfont\tiny\mbox{#1}}}{\geq}}

\newcommand{\reals}{{\mathbb{R}}}

\newcommand{\naturals}{{\mathbb{N}}}

\newcommand{\identityf}[1]{\mathbbm{1} \left\{#1\right\}}
\newcommand{\bfu}[1]{\underline{\bm{#1}}}

\newcommand{\prob}[1]{\mathbb{P}\left\{{#1}\right\}}
\newcommand{\mean}[1]{\mathbb{E}\left[{#1}\right]}

\newcommand{\means}[2]{\mathbb{E}_{#1}\left[{#2}\right]}

\newcommand{\mtx}[1]{\mathbf{#1}}
\newcommand{\vc}[1]{\mathbf{#1}}
\newcommand{\norm}[1]{\|{#1}\|}
\newcommand{\pnorm}[2]{\|{#1}\|_{#2}}
\newcommand{\tp}[1]{{#1}^{\top}}

\newcommand{\proj}[2]{\Pi_{#1}\left(#2\right)}

\newcommand{\tuple}[1]{\mathsf{#1}}
\newcommand{\calA}{{\cal A}}
\newcommand{\calB}{{\cal B}}
\newcommand{\calC}{{\cal C}}
\newcommand{\calD}{{\cal D}}

\newcommand{\calF}{{\cal F}}
\newcommand{\calG}{{\cal G}}
\newcommand{\calU}{{\cal U}}

\newcommand{\calI}{{\cal I}}
\newcommand{\calJ}{{\cal J}}

\newcommand{\calR}{{\cal R}}

\newcommand{\calS}{{\cal S}}
\newcommand{\calL}{{\cal L}}
\newcommand{\calH}{{\cal H}}
\newcommand{\calX}{{\cal X}}

\newcommand{\calT}{{\cal T}}
\newcommand{\calP}{{\cal P}}

\DeclareMathOperator*{\argmax}{arg\,max}
\DeclareMathOperator*{\argmin}{arg\,min}

\usepackage[capitalize,noabbrev]{cleveref}

\theoremstyle{plain}
\newtheorem{theorem}{Theorem}
\newtheorem*{theorem*}{Theorem}
\newtheorem{proposition}{Proposition}
\newtheorem*{proposition*}{Proposition}
\newtheorem{lemma}{Lemma}
\newtheorem{corollary}{Corollary}
\newtheorem*{corollary*}{Corollary}

\theoremstyle{remark}

\theoremstyle{definition}
\newtheorem{definition}{Definition}

\usepackage[textsize=tiny]{todonotes}

\icmltitlerunning{On the Robustness of Randomized Ensembles to Adversarial Perturbations}

\begin{document}

\twocolumn[
\icmltitle{On the Robustness of Randomized Ensembles to Adversarial Perturbations}



\icmlsetsymbol{dagger}{$\dagger$}
\begin{icmlauthorlist}
\icmlauthor{Hassan Dbouk}{uiuc}
\icmlauthor{Naresh R. Shanbhag}{uiuc}
\end{icmlauthorlist}

\icmlaffiliation{uiuc}{Department of Electrical and Computer Engineering, University of Illinois at Urbana-Champaign, Urbana, USA}
%\icmlaffiliation{qualcomm}{Qualcomm AI Research, San Diego, USA}
%\icmlaffiliation{dagger}{d AI Research, San Diego, USA}
%\icmlaffiliation{sch}{School of ZZZ, Institute of WWW, Location, Country}

\icmlcorrespondingauthor{Hassan Dbouk}{hdbouk2@illinois.edu}
%\icmlcorrespondingauthor{Firstname2 Lastname2}{first2.last2@www.uk}

\icmlkeywords{Machine Learning, ICML}

\vskip 0.3in
]



\printAffiliationsAndNotice{}  % leave blank if no need to mention equal contribution
%\printAffiliationsAndNotice{$\dagger$ This work was performed when Hassan Dbouk was with the University of Illinois at Urbana-Champaign prior to joining Qualcomm.} % otherwise use the standard text.

\begin{abstract}
Randomized ensemble classifiers (RECs), where \emph{one} classifier is randomly selected during inference, have emerged as an attractive alternative to traditional ensembling methods for realizing adversarially robust classifiers with limited compute requirements. However, recent works have shown that existing methods for constructing RECs are more vulnerable than initially claimed, casting major doubts on their efficacy and prompting \emph{fundamental} questions such as: \say{\emph{When are RECs useful?}}, \say{\emph{What are their limits?}}, and \say{\emph{How do we train them?}}. In this work, we first demystify RECs as we derive fundamental results regarding their theoretical limits, necessary and sufficient conditions for them to be useful, and more. Leveraging this new understanding, we propose a new boosting
algorithm (BARRE) for training robust RECs, and empirically demonstrate its effectiveness at defending against strong $\ell_\infty$ norm-bounded adversaries across various network architectures and datasets. Our code can be found at \url{https://github.com/hsndbk4/BARRE}.
\end{abstract}

\section{Introduction}

Defending deep networks against adversarial perturbations \citep{szegedy2013intriguing,biggio2013evasion,goodfellow2014explaining} remains a difficult task. Several proposed defenses \citep{papernot2016distillation,pang2019improvingADP,yang2019me,sen2019empir,pinot2020randomization} have been subsequently \say{broken} by stronger adversaries \citep{carlini2017towards,athalye2018obfuscated,tramer2020adaptive,dbouk2022adversarial}, whereas strong defenses \citep{cisse2017parseval,tramer2018ensemble,cohen2019certified}, such as adversarial training (AT) \citep{goodfellow2014explaining,trades,madry2018towards}, achieve unsatisfactory levels of robustness\footnote{when compared to the high clean accuracy achieved in a non-adversarial setting}.

\begin{figure}[bt]
  \centering
    \includegraphics[width=0.9\columnwidth]{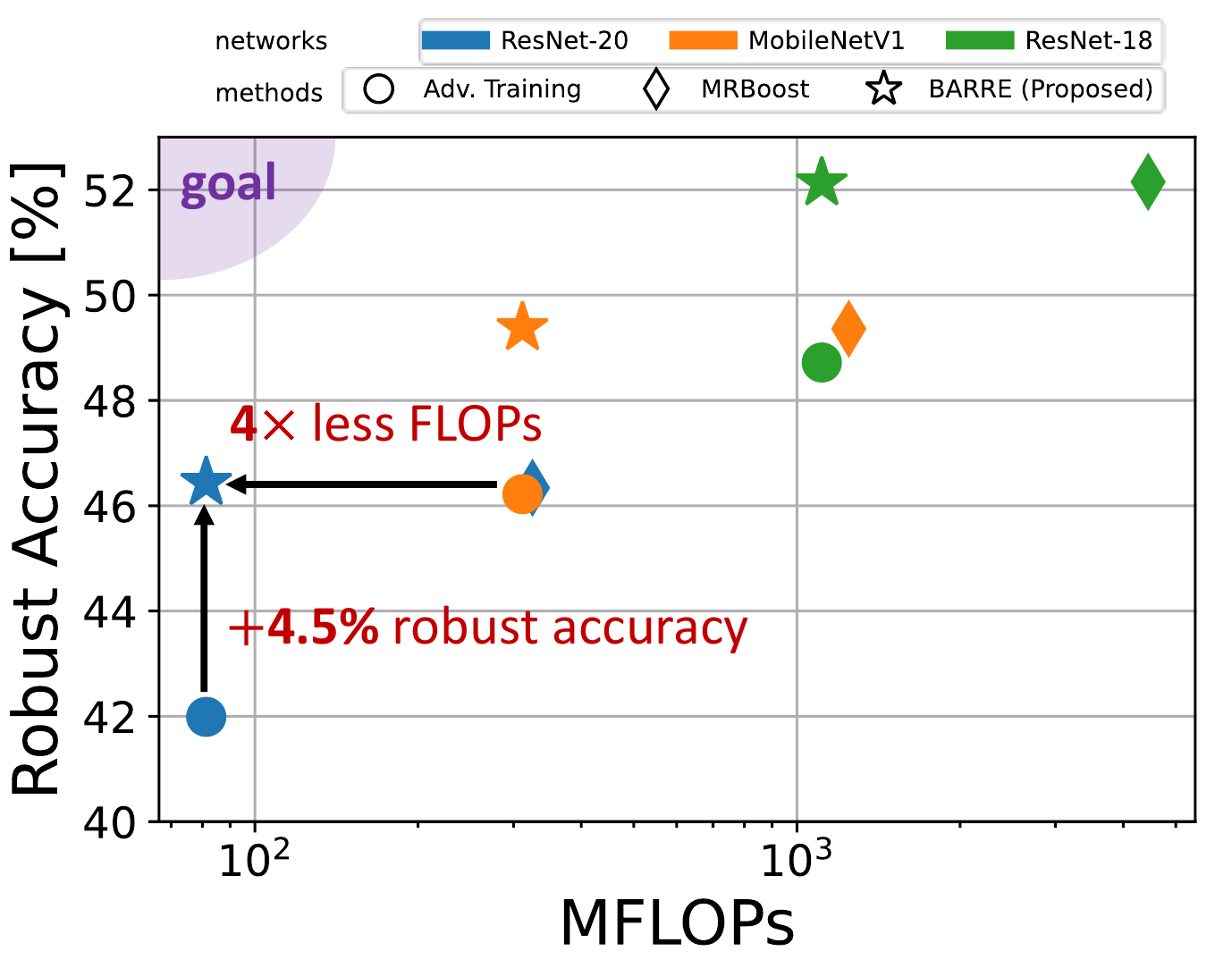}
    \caption{The efficacy of employing randomized ensembles ($\star$) for achieving robust and efficient inference compared to AT ($\bullet$) and deterministic ensembling MRBoost ($\blacklozenge$) on CIFAR-10. Robustness is measured using the standard $\ell_\infty$ norm-bounded adversary with radius $\epsilon=\nicefrac{8}{255}$.}
    \label{fig:motivation}
\end{figure}

A popular belief in the adversarial community is that single model defenses, e.g., AT, lack the capacity to defend against all possible perturbations, and that constructing an ensemble of diverse, often smaller, models should be more cost-effective  \citep{pang2019improvingADP,kariyappa2019improvingGAL,pinot2020randomization,yang2020dverge,yang2021trs,abernethy2021multiclass,mrboost}. Indeed, recent \emph{deterministic} robust ensemble methods, such as MRBoost \citep{mrboost}, have been successful at achieving higher robustness compared to AT baselines using the same network architecture, at the expense of $4\times$ more compute (see Fig.~\ref{fig:motivation}). In fact, Fig~\ref{fig:motivation} indicates that one can simply adversarially train \emph{larger} deep nets that can match the robustness and compute requirements of MRBoost models, rendering state-of-the-art boosting techniques obsolete for designing classifiers that are both \emph{robust} and \emph{efficient}.   

In contrast, \emph{randomized} ensembles, where \underline{one} classifier is randomly selected during inference, offer a unique way of ensembling that can operate with limited compute resources. However, the recent work of Dbouk \& Shanbhag \yrcite{dbouk2022adversarial} has cast major concerns regarding their efficacy, as they successfully compromised the state-of-the-art randomized defense of Pinot et al. \yrcite{pinot2020randomization} by large margins using their proposed ARC adversary. Furthermore, there is an apparent lack of \emph{proper} theory on the robustness of randomized ensembles, as fundamental questions such as: \say{when does randomization help?} or \say{how to find the optimal sampling probability?} remain unanswered.

\textbf{Contributions}. In this work, we first provide a theoretical framework for analyzing the adversarial robustness of randomized ensmeble classifiers (RECs). Our theoretical results enable us to better understand randomized ensembles, revealing interesting and useful answers regarding their limits, necessary and sufficient conditions for them to be useful, and efficient methods for finding the optimal sampling probability. 
Next, guided by our threoretical results, we propose BARRE, a boosting algorithm to construct randomized ensemble classifiers that achieve robustness on par with that of AT and MRBoost but at a \emph{fraction of their computational cost} (see Fig.~\ref{fig:motivation}). 
%achieving state-of-the-art robustness.
We validate the effectiveness of BARRE via comprehensive experiments across multiple network architectures and datasets. 

\section{Background and Related Work} 

\textbf{Adversarial Robustness}. Deep neural networks are known to be vulnerable to adversarial perturbations \citep{szegedy2013intriguing,biggio2013evasion}. In an attempt to robustify deep nets, several defense methods have been proposed \citep{katz2017reluplex,madry2018towards,cisse2017parseval,trades,yang2020dverge,mrboost,tjeng2018evaluating,xiao2018training,raghunathan2018certified,yang2020randomized}. While some heuristic-based \emph{empirical} defenses have later been broken by better adversaries \citep{carlini2017towards,athalye2018obfuscated,tramer2020adaptive}, strong defenses, such as adversarial training (AT) \citep{goodfellow2014explaining,madry2018towards,trades}, remain unbroken but achieve unsatisfactory levels of robustness.

\textbf{Ensemble Defenses}. Building on the massive success of classic ensemble methods in machine learning \citep{breiman1996bagging,freund1997decision,dietterich2000ensemble}, robust ensemble methods \citep{kariyappa2019improvingGAL, pang2019improvingADP, sen2019empir,yang2020dverge,yang2021trs,abernethy2021multiclass,mrboost} have emerged as a natural solution to compensate for the unsatisfactory performance of existing single-model defenses, such as AT. Earlier works \citep{kariyappa2019improvingGAL, pang2019improvingADP, sen2019empir} relied on heuristic-based techniques for inducing diversity within the ensembles, and have been subsequently shown to be weak \citep{tramer2020adaptive,athalye2018obfuscated}. Recent methods, such as RobBoost \citep{abernethy2021multiclass} and MRBoost \citep{mrboost},  formulate the design of robust ensembles from a margin boosting perspective, achieving state-of-the-art robustness for deterministic ensemble methods. This achievement comes at a massive ($4-5\times$) increase in compute requirements, as each inference requires executing all members of the ensemble, deeming them unsuitable for safety-critical edge applications \citep{NAS,sehwag2020hydra,dbouk2021generalized}. Randomized ensembles \citep{pinot2020randomization}, where one classifier is chosen randomly during inference, offer a more compute-efficient alternative. However, this defense has been recently broken independently by Dbouk \& Shanbhag \yrcite{dbouk2022adversarial} and Zhang et al. \yrcite{mrboost}.
%However, their ability to defend against strong adversaries remains unclear \citep{dbouk2022adversarial,mrboost}. 
In this work, we develop theoretical results to delineate scenarios when randomized ensembles can be effective at defending against adversarial perturbations, and propose a boosting algorithm for training such ensembles to achieve high levels of robustness with limited compute requirements.
%In this work, we show that randomized ensemble classifiers can be effective at defending against adversarial perturbations, and propose a boosting algorithm for training such ensembles, thereby achieving high levels of robustness with limited compute requirements.

\textbf{Randomized Defenses}. A randomized defense, where the defender adopts a random strategy for classification, is intuitive: if the defender does not know what is the exact policy used for a certain input, then one expects that the adversary will struggle \emph{on average} to fool such a defense. Theoretically, Bayesian Neural Nets (BNNs) \citep{neal2012bayesian} have been shown to be robust (in the large data limit) to gradient-based attacks \citep{carbone2020robustness}, whereas  Pinot et al. \yrcite{pinot2020randomization} has shown that a randomized ensemble classifier (REC) with higher robustness exists for every deterministic classifier.
%, under a deterministic attacker. 
However, realizing strong and practical randomized defenses remains elusive 
as BNNs are too computationally prohibitive and existing methods \citep{xie2018mitigating,dhillon2018stochastic,yang2019me} often end up being compromised by adaptive attacks \citep{athalye2018obfuscated,tramer2020adaptive}. Even BAT, the proposed method of Pinot et al. \yrcite{pinot2020randomization} for robust RECs, was recently broken by Zhang et al. \yrcite{mrboost} and Dbouk \& Shanbhag \yrcite{dbouk2022adversarial}. In contrast, our work first demystifies randomized ensembles as we derive fundamental results regarding the limit of RECs, necessary and sufficient conditions for them to be useful, and efficient methods for finding the optimal sampling probability. Empirically, our proposed boosting algorithm (BARRE) can successfully train RECs, to achieve both robust and efficient classification.%achieving state-of-the-art robustness for RECs. 
%\hd{I am trying to shy away from using state-of-the art too much}

\section{Preliminaries \& Problem Setup}
\textbf{Notation}. Let $\calF = \{f_1,...,f_M\}$ be a collection of $M$ arbitrary $C$-ary classifiers $f_i: \reals^d \rightarrow [C]$. A soft classifier, denoted by $\tilde{f}: \reals^d \rightarrow \reals^C$, can be used to construct a hard classifier $f(\vc{x}) = \argmax_{c\in[C]} [\tilde{f}(\vc{x})]_c$, where $[\vc{v}]_c = v_c$. We use the notation $f(\cdot|\bm{\theta})$ to represent \emph{parametric} classifiers  where $f$ is a \emph{fixed} mapping and $\bm{\theta}\in\Theta$ represents the learnable parameters. Let $\Delta_M = \{\vc{v}\in[0,1]^M: \sum v_i =1\}$ be the probability simplex of dimension $M-1$. Given a probability vector $\bm{\alpha} \in \Delta_M$, we construct a randomized ensemble classifier (REC) $f_{\bm{\alpha}}$ such that $f_{\bm{\alpha}}(\vc{x}) = f_i(\vc{x})$ with probability $\alpha_i$. In contrast, traditional ensembling methods construct a deterministic ensemble classifier (DEC) using the soft classifiers as follows\footnote{the normalizing constant $\frac{1}{M}$ does not affect the classifier output}: $\bar{f}(\vc{x}) = \argmax_{c\in [C]}[\sum_{i=1}^M \tilde{f}_i(\vc{x})]_c$. Denote $\tuple{z}=(\vc{x},y)\in \reals^d \times [C]$ as a feature-label pair that follows some unknown distribution $\calD$. Let $\calS \subset \reals^d$ be a closed and bounded set representing the attacker's perturbation set. A typical choice of $\calS$ in the adversarial community is the $\ell_p$ ball of radius $\epsilon$: $\calB_p(\epsilon)=\{\bm{\delta}\in\reals^d: \pnorm{\bm{\delta}}{p}\leq \epsilon\}$. 
For a classifier $f_i\in\calF$ and data-point $\tuple{z}=(\vc{x},y)$, define $\calS_i(\tuple{z})=\{\bm{\delta}\in\calS: f_i(\vc{x}+\bm{\delta})\neq y\}$ to be the set of valid \emph{adversarial} perturbations to $f_i$ at $\tuple{z}$.

\begin{definition}
\label{def:adv-risk}
For any (potentially random) classifier $f$, define the \emph{adversarial risk} $\eta$:
\begin{equation}
    \eta(f) = \means{\tuple{z}\sim \calD}{\max_{\bm{\delta}\in\calS}{ \means{f}{\identityf{f(\vc{x}+\bm{\delta})\neq y}}}}
\end{equation}
\end{definition}
The adversarial risk measures the robustness of $f$ on average in the presence of an adversary (attacker) restricted to the set $\calS$. For the special case of $\calS=\{\vc{0}\}$, the adversarial risk reduces to the \emph{standard risk} of $f$:
\begin{equation}
    \eta_0(f) = \means{\tuple{z}\sim \calD}{ \means{f}{\identityf{f(\vc{x})\neq y}}} = \prob{f(\vc{x})\neq y}
\end{equation}
The more commonly reported robust accuracy of $f$, i.e., accuracy against adversarially perturbed inputs, can be directly computed from $\eta(f)$. The same can be said for the clean accuracy and $\eta_0(f)$.

When working with an REC $f_{\bm{\alpha}}$, the adversarial risk can be expressed as:
\begin{equation}\label{eq:eta-1}
    \eta(f_{\bm{\alpha}})  =\means{\tuple{z}\sim \calD}{\max_{\bm{\delta}\in\calS}{ \sum_{i=1}^M \alpha_i \identityf{f_i(\vc{x}+\bm{\delta})\neq y}}}
\end{equation}
where we use the notation $\eta(f_{\bm{\alpha}}) \equiv \eta(\bm{\alpha})$ whenever the collection $\calF$ is fixed. Let $\{\vc{e}_i\}_{i=1}^M\subset\{0,1\}^M$ be the standard basis vectors of $\reals^M$, then we employ the notation $\eta(f_i) = \eta(f_{\vc{e}_i}) \equiv \eta(\vc{e}_i) = \eta_i$.

\section{The Adversarial Risk of a Randomized Ensemble Classifier} \label{sec:theory}
In this section, we develop our main theoretical findings regarding the adversarial robustness of any randomized ensemble classifier. Detailed proofs of all statements and theorems can be found in Appendix~\ref{app:proofs}.
\subsection{Properties of \texorpdfstring{$\eta$}{}}
We start with the following statement:

\begin{proposition}
\label{prop:eta} 
For any $\calF=\{f_i\}_{i=1}^M$, perturbation set $\calS \subset \reals^d$, and data distribution $\calD$, the adversarial risk $\eta$ is a piece-wise linear convex function $\forall \bm{\alpha} \in \Delta_M$. Specifically, $\exists K\in \naturals$ configurations $\calU_k\subseteq\{0,1\}^M$ $\forall k\in[K]$ and p.m.f. $\vc{p}\in \Delta_K$ such that:
\begin{equation}\label{eq:eta-2}
    \eta(\bm{\alpha}) = \sum_{k=1}^K\left(p_k \cdot \max_{\vc{u}\in \calU_k}\left\{\tp{\vc{u}}\bm{\alpha}\right\}\right)
\end{equation}
\end{proposition}
Before we explain the intuition behind Proposition~\ref{prop:eta}, we first make the following observations: 

\textbf{Generality}. Proposition~\ref{prop:eta} makes no assumptions about the classifiers $\calF$, i.e., it applies even to the enigmatic deep nets. While the majority of theoretical results in the literature have been restricted to $\ell_p$-bounded adversaries, Proposition~\ref{prop:eta} holds for any closed and bounded perturbation set $\calS$. This is crucial, as real-world attacks are often not restricted to  $\ell_p$ balls around the input \citep{liu2018dpatch,duan2020adversarial}. This generality is further inherited by all of our results, as they build on Proposition~\ref{prop:eta}.

\textbf{Analytic Form}. Proposition~\ref{prop:eta} allows us to re-write the adversarial risk in \eqref{eq:eta-1} using the analytic form in \eqref{eq:eta-2}, which is much simpler to analyze and work with. In fact, the analytic form in \eqref{eq:eta-2} enables us to derive our main theoretical results in Sections~\ref{ssec:two}
\&~\ref{ssec:bounds}, which include \emph{tight fundamental bounds} on $\eta$.

\textbf{Optimal Sampling}. The convexity of $\eta$ implies that any local minimum $\bm{\alpha}^*$ is also a global minimum. The probability simplex is a closed convex set, thus a global minimum, which need not be unique, is always achievable. Since $\eta$ is piece-wise linear, then there always exists a \emph{finite} set of candidate solutions for $\bm{\alpha}^*$. For $M\leq 3$, we  efficiently enumerate all candidates in Section~\ref{ssec:two}, eliminating the need for any sophisticated search method. For larger $M$ however, enumeration becomes intractable. In Section~\ref{ssec:sampling}, we construct an optimal algorithm for finding $\bm{\alpha}^*$ by leveraging the classic sub-gradient method \citep{shor2012minimization} for optimizing sub-differentiable functions.

\begin{figure}[t]
  \centering
    \includegraphics[width=\columnwidth]{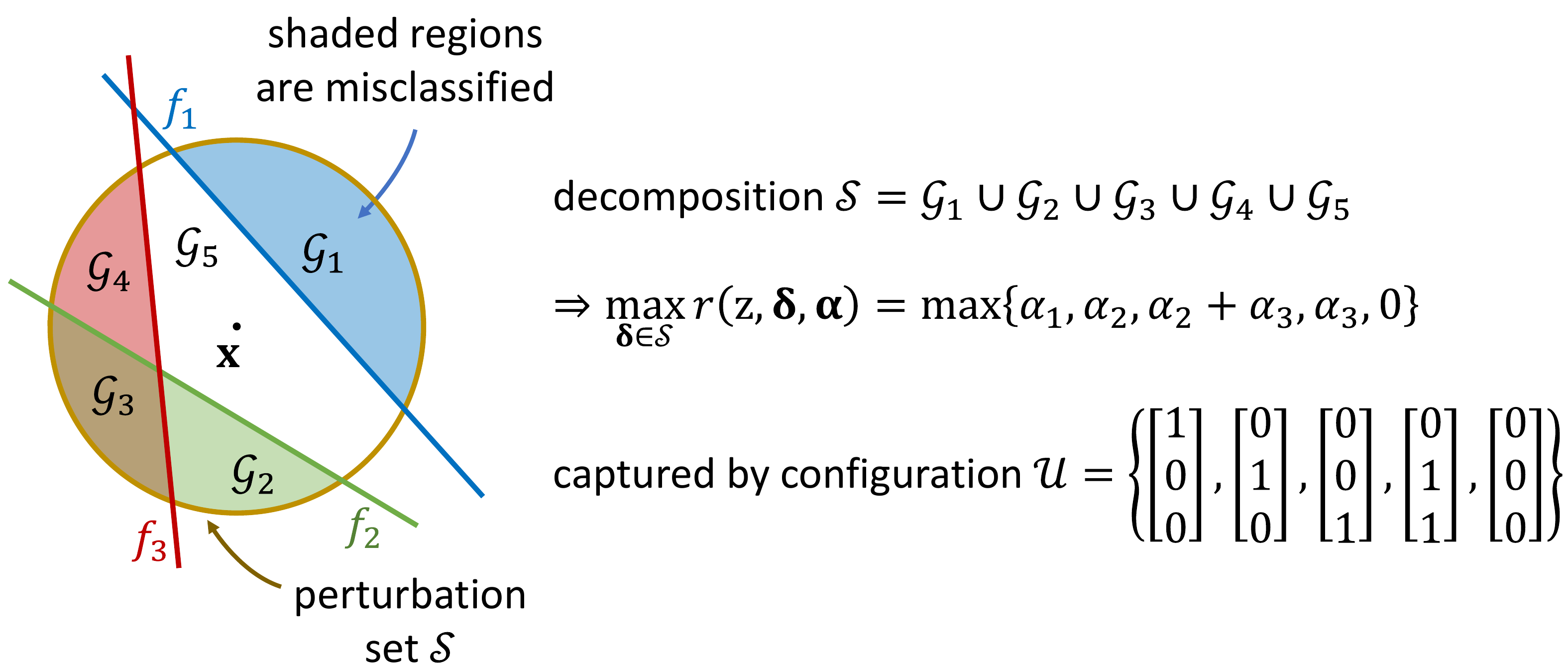}
    \caption{Illustration of the equivalence in \eqref{eq:singe-risk-equiv} using an example of three classifiers in $\reals^2$. The shaded areas represent regions in the attacker-restricted input space where each classifier makes an error. All classifiers correctly classify $\vc{x}$. The set $\calU$ uniquely captures the interaction between $\tuple{z}$ and $f_1$, $f_2$, \& $f_3$ inside $\calS$.}
    \label{fig:illustration}
\end{figure}

\textbf{Intuition}. Consider a data-point $\tuple{z}\in \reals^d \times [C]$, then for any $\bm{\delta} \in \calS$ and $\bm{\alpha}\in \Delta_M$ we have the per-sample risk:
\begin{equation}\label{eq:per-sample-risk}
    r\left(\tuple{z},\bm{\delta},\bm{\alpha}\right)= \sum_{i=1}^M \alpha_i \identityf{f_i(\vc{x}+\bm{\delta})\neq y}  = \tp{\vc{u}}\bm{\alpha}
\end{equation}
where $\vc{u}\in\{0,1\}^M$ such that $u_i=1$ if and only if $\bm{\delta}$ is adversarial to $f_i$ at $\tuple{z}$. Since $\vc{u}$ is independent of $\bm{\alpha}$, we thus obtain a many-to-one mapping from $\bm{\delta}\in \calS$ to $\vc{u}\in\{0,1\}^M$. Therefore, for any $\bm{\alpha}$ and $\tuple{z}$, we can always decompose the perturbation set $\calS$, i.e., $\calS = \calG_1 \cup ... \cup \calG_n$, into $n\leq 2^M$ subsets, such that: $\forall \bm{\delta} \in \calG_j: r\left(\tuple{z},\bm{\delta},\bm{\alpha}\right) = \tp{\bm{\alpha}}\vc{u}_j$
for some binary vector $\vc{u}_j$ independent of $\bm{\alpha}$. Let $\calU=\{\vc{u}_j\}_{j=1}^{n}$ be the collection of these vectors, then we can write:
\begin{align} \label{eq:singe-risk-equiv}
    \begin{split}
        \max_{\bm{\delta}\in\calS} r\left(\tuple{z},\bm{\delta},\bm{\alpha}\right) &= \max_{\bm{\delta}\in\calG_1 \cup ... \cup \calG_n} r\left(\tuple{z},\bm{\delta},\bm{\alpha}\right) \\
        &=\max_{j \in [n]} \left\{ \max_{\bm{\delta}\in\calG_j} r\left(\tuple{z},\bm{\delta},\bm{\alpha}\right)\right\} \\
        &= \max_{\vc{u}\in \calU} \left\{ \tp{\vc{u}}\bm{\alpha}\right\}
    \end{split}
\end{align}
The main idea behind the equivalence in \eqref{eq:singe-risk-equiv} is that we can represent any configuration of classifiers, data-point and perturbation set using a unique set of binary vectors $\calU$. For example, Fig.~\ref{fig:illustration} pictorially depicts this equivalence using a case of $M=3$ classifiers in $\reals^2$ with $\calS=\calB_2(\epsilon)$. This equivalence is the key behind Proposition~\ref{prop:eta}, since the point-wise max term in \eqref{eq:singe-risk-equiv} is piece-wise linear and convex $\forall \bm{\alpha}\in \Delta_M$. Finally, Proposition~\ref{prop:eta} holds due to the pigeon-hole principle and the linearity of expectation.

\begin{figure*}[t]
  \centering
    \includegraphics[width=1.8\columnwidth]{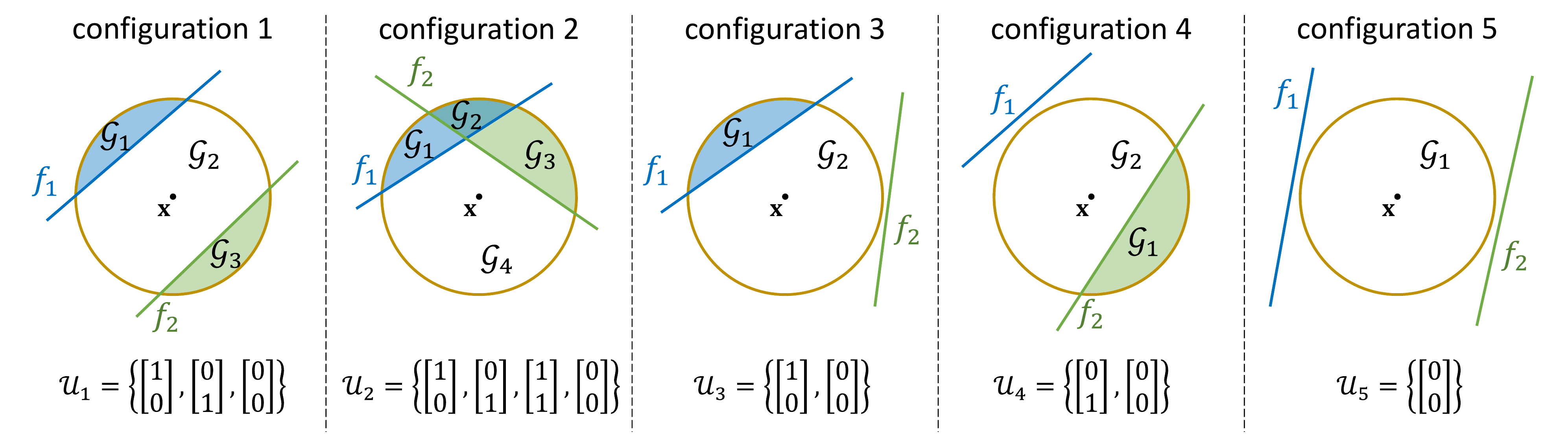}
    \caption{Enumeration of all $K=5$ unique configurations with two classifiers and a data-point around a set $\calS$. Note that since $\alpha_i \geq 0 \ \forall i$, the $\vc{0}$ vector is redundant in $\calU_k$ for $k\in [4]$, which explains why $K=5$ and not more. %\ns{is there an expression or an upper bound for $K$ in terms of $M$? Is $n=4$ in this example? Can mention this too.}
    }
    \label{fig:enumerate}
\end{figure*}

\subsection{Special Case of Two Classifiers}\label{ssec:two}
With two classifiers only, we can leverage the analytic form of $\eta$ in \eqref{eq:eta-2} and enumerate all possible classifiers/data-point configurations around $\calS$ by enumerating all configurations $\calU_k \subseteq \{0,1\}^2$. Specifically, Fig.~\ref{fig:enumerate} visualizes all $K=5$ such unique configurations, which allows us to write $\forall \bm{\alpha}\in \Delta_2$:
\begin{equation}\label{eq:eta-m2}
    \eta(\bm{\alpha}) = p_1 \cdot \max\{\alpha_1,\alpha_2\} + p_2 \cdot 1 + p_3 \cdot \alpha_1 + p_4 \cdot \alpha_2 + p_5 \cdot 0
\end{equation}
where $\vc{p} \in \Delta_5$ is the p.m.f. of \say{binning} any data-point $\tuple{z}$ into any of the five configurations, under the data distribution $\tuple{z}\sim\calD$. Using \eqref{eq:eta-m2}, we obtain the following result:

\begin{theorem}
\label{thm:two-classifiers} 

For any two classifiers $f_1$ and $f_2$ with individual adversarial risks $\eta_1$ and $\eta_2$, respectively, subject to a perturbation set $\calS \subset \reals^d$ and data distribution $\calD$, if:
\begin{equation}
    \label{eq:thm1-cond}
    \prob{\tuple{z}\in\calR_1} > |\eta_1-\eta_2|
\end{equation}
where:
\begin{equation} \label{eq:R_1}
    \calR_1 = \{\tuple{z}\in\reals^d \times [C]: \calS_1(\tuple{z}), \calS_2(\tuple{z})\neq \varnothing, \calS_1(\tuple{z})\cap \calS_2(\tuple{z})=\varnothing\}
\end{equation}
then the optimal sampling probability $\bm{\alpha}^*=\tp{\left[\nicefrac{1}{2}\  \nicefrac{1}{2}\right]}$ uniquely minimizes $\eta(\bm{\alpha})$ resulting in $\eta(\bm{\alpha}^*)=\frac{1}{2}\left(\eta_1+\eta_2-\prob{\tuple{z}\in \calR_1}\right)$. Otherwise,  $\bm{\alpha}^* \in \{\vc{e}_1, \vc{e}_2\}$ minimizes $\eta(\bm{\alpha})$, where $\vc{e}_i$s are the standard basis vectors of $\reals^2$.
\end{theorem}
Theorem~\ref{thm:two-classifiers} provides us with a \emph{complete} description of how randomized ensembles operate when $M=2$. We discuss its implications below:

\textbf{Interpretation}. Theorem~\ref{thm:two-classifiers} states that randomization is \emph{guaranteed} to help when the condition in \eqref{eq:thm1-cond} is satisfied, i.e., when the probability of data-points $\tuple{z}$ ($\prob{\tuple{z}\in\calR_1}$) for which it is possible to find adversarial perturbations that can fool $f_1$ or $f_2$ but not both (see configuration 1 in Fig.~\ref{fig:enumerate}), is greater than the absolute difference ($|\eta_1-\eta_2|$) of the individual classifiers' adversarial risks. Consequently, if the adversarial risks of the classifiers are heavily skewed, i.e., $|\eta_1 - \eta_2|$ is large, then randomization is less likely to help, since condition \eqref{eq:thm1-cond} becomes harder to satisfy. This, in fact, is the case for BAT defense \citep{pinot2020randomization} since it generates two classifiers with $\eta_1 <1$ and $\eta_2 = 1$. Theorem~\ref{thm:two-classifiers} indicates that adversarial defenses should strive to achieve $\eta_1 \approx \eta_2$ for randomization to be effective. In practice, it is very difficult to make $\prob{\tuple{z}\in\calR_1}$ very large compared to $\eta_1$ and $\eta_2$ due to transferability of adversarial perturbations.

\textbf{Optimality Condition}. In fact, the condition in \eqref{eq:thm1-cond} is actually a \emph{necessary} and \emph{sufficient} condition for $\eta(\bm{\alpha}^*) < \min\{\eta_1,\eta_2\}$. That is, a randomized ensemble of $f_1$ and $f_2$ is \emph{guaranteed} to achieve smaller adversarial risk than either $f_1$ of $f_2$ if and only if \eqref{eq:thm1-cond} holds. This also implies that it is \emph{impossible} to have a nontrivial\footnote{that is different than $\vc{e}_1$ or $\vc{e}_2$} unique global minimizer other than $\bm{\alpha}^*=\tp{\left[\nicefrac{1}{2}\  \nicefrac{1}{2}\right]}$, which provides further theoretical justification for why the BAT defense \citep{pinot2020randomization} does not work, where $\bm{\alpha}^*=\tp{[0.9\ 0.1]}$ was claimed to be a unique optimum (obtained via sweeping $\bm{\alpha}$).

\textbf{Theoretical Limit}. From Theorem~\ref{thm:two-classifiers}, we can directly obtain a \emph{tight} bound on the adversarial risk: 
\begin{corollary}\label{cor:limit} For any two classifiers $f_1$ and $f_2$ with individual adversarial risks $\eta_1$ and $\eta_2$, respectively, perturbation set $\calS$, and data distribution $\calD$: 
\begin{equation}\label{eq:thm1-limit}
   \min_{\bm{\alpha}\in \Delta_2} \eta(\bm{\alpha})\geq \min\left\{\frac{1}{2} \max\{\eta_1, \eta_2\}, \min\{\eta_1,\eta_2\}\right\}
\end{equation}
\end{corollary}
In other words, it is impossible for a REC with $M=2$ classifiers to achieve a risk smaller than the RHS in \eqref{eq:thm1-limit}. In the next section, we derive a more general version of this bound for arbitrary $M$.

\textbf{Simplified Search}. Theorem~\ref{thm:two-classifiers} \emph{eliminates} the need for sweeping $\bm{\alpha}$ to find the optimal sampling probability $\bm{\alpha}^*$ when working with $M=2$ classifiers as done in Pinot et al. \yrcite{pinot2020randomization} and Dbouk \& Shanbhag \yrcite{dbouk2022adversarial}. We only need to evaluate $\eta\left(\tp{\left[\nicefrac{1}{2}\  \nicefrac{1}{2}\right]}\right)$ and check if it is smaller than $\min\{\eta_1, \eta_2\}$ to choose our optimal sampling probability. Thus, the defender's optimal strategy is to either sample between the classifiers uniformly at random or determinstically choose one of the classifiers. This observation might sound counter-intuitive at first, as one would expect a more \say{nuanced} approached to sampling based on the classifiers' relative performances. Interestingly, Vorobeychik \& Li \yrcite{vorobeychik2014optimal} derive a similar result for $M=2$ for a different problem of an adversary attempting to reverse engineer the defender's classifier via queries.

\textbf{Extension to Three Classifiers}. In fact, a simplified search strategy for the special case of $M = 3$ can also be derived in a similar fashion, as shown below:

\begin{theorem}\label{thm:three-class} Define $\calA \subset \Delta_3$ to be the set of the following vectors:
\begin{align}
\begin{split}
    \calA &= \left\{ \begin{bmatrix}
1\\
0 \\
0
\end{bmatrix}, \begin{bmatrix}
0\\
1 \\
0
\end{bmatrix},
\begin{bmatrix}
0\\
0 \\
1
\end{bmatrix},
\begin{bmatrix}
\nicefrac{1}{2}\\
\nicefrac{1}{2} \\
0
\end{bmatrix},
\begin{bmatrix}
0\\
\nicefrac{1}{2} \\
\nicefrac{1}{2}
\end{bmatrix},
\begin{bmatrix}
\nicefrac{1}{2} \\
0\\
\nicefrac{1}{2}
\end{bmatrix},\right. \\
& \left. \quad
\begin{bmatrix}
\nicefrac{1}{2} \\
\nicefrac{1}{4}\\
\nicefrac{1}{4}
\end{bmatrix}
,
\begin{bmatrix}
\nicefrac{1}{4} \\
\nicefrac{1}{2}\\
\nicefrac{1}{4}
\end{bmatrix},
\begin{bmatrix}
\nicefrac{1}{4} \\
\nicefrac{1}{4}\\
\nicefrac{1}{2}
\end{bmatrix},
\begin{bmatrix}
\nicefrac{1}{3} \\
\nicefrac{1}{3}\\
\nicefrac{1}{3}
\end{bmatrix} \right\}
\end{split}
\end{align}
Then for any three classifiers $f_1$, $f_2$, and $f_3$, perturbation set $\calS \subset \reals^d$, and data distribution $\calD$, we have:
\begin{equation}\label{eq:thm3-min}
    \min_{\bm{\alpha} \in \Delta_3} \eta(\bm{\alpha})  =  \min_{\bm{\alpha} \in \calA} \eta(\bm{\alpha})
\end{equation}
The set $\calA$ is optimal, in the sense that there exist no smaller set $\calA'$ such that \eqref{eq:thm3-min} holds.
\end{theorem}
%\ns{Theorem refers to (75) in Appendix. Please fix.}

\subsection{Tight Fundamental Bounds} \label{ssec:bounds}
A fundamental question remains to be answered: given an ensemble $\calF$ of $M$ classifiers with adversarial risks $\eta_1, ..., \eta_M$, what is the tightest bound we can provide for the adversarial risk $\eta(\bm{\alpha})$ of a randomized ensemble constructed from $\calF$? The following theorem answers this question:
\begin{theorem}
\label{thm:bounds} For a perturbation set $\calS$, data distribution $\calD$, and collection of $M$ classifiers $\calF$ with individual adversarial risks $\eta_i$ ($i\in [M]$) such that $0<\eta_1\leq ... \leq \eta_M\leq 1$, we have $\forall \bm{\alpha}\in \Delta_M$:
\begin{equation}\label{eq:thm2-bound}
       \min_{k\in [M]} \left\{\frac{\eta_k}{k}\right\}  \leq  \eta(\bm{\alpha}) \leq \eta_M  
\end{equation}
Both bounds are tight in the sense that if all that is known about the setup $\calF$, $\calD$, and $\calS$ is $\{\eta_i\}_{i=1}^M$, then there exist no tighter bounds. Furthermore, the upper bound is always met if $\bm{\alpha}=\vc{e}_M$, and the lower bound (if achievable) can be met if $\bm{\alpha}=\tp{\left[\frac{1}{m}\ ...\ \frac{1}{m}\ 0\ ...\ 0\right]}$, where $m=\argmin_{k\in [M]} \{\frac{\eta_k}{k}\}$.
\end{theorem}
\textbf{Upper bound}: The upper bound in \eqref{eq:thm2-bound} holds due to the convexity of $\eta$ (Proposition~\ref{prop:eta}) and the fact $\Delta_M = \calH\left(\{\vc{e}_i\}_{i=1}^M\right)$, where $\calH(\calX)$ is the convex hull of the set of points $\calX$. 

\textbf{Implications of upper bound}: Intuitively, we expect that a randomized ensemble cannot be worse than the worst performing member (in this case $f_M$). A direct implication of this is that if all the members have similar robustness $\eta_i \approx \eta_j\ \forall i,j$, then randomized ensembling is \emph{guaranteed} to either improve or achieve the same robustness. In contrast, deterministic ensemble methods that average logits \citep{mrboost,abernethy2021multiclass,kariyappa2019improvingGAL} do not even satisfy this upper bound (see Appendix~\ref{app:upper-bound}). In other words, there are \emph{no} worst-case performance guarantees with deterministic ensembling, even if all the classifiers are robust. Note that this does not imply that deterministic ensembling methods are inherently more vulnerable.

\textbf{Lower bound}: The main idea behind the proof of the lower bound in \eqref{eq:thm2-bound} is to show that $\forall \bm{\alpha}\in \Delta_M$:
\begin{align}
    \begin{split}
     \eta(\bm{\alpha}) &\geq \sum_{i=1}^M \left(\left( \eta_i - \eta_{i-1}\right) \cdot \max_{j\in \{i,...,M\}}\{\alpha_j\} \right) \\
     &= h(\bm{\alpha}) \geq \min_{\bm{\alpha}\in \Delta_M} h(\bm{\alpha}) = h(\bm{\alpha}^*) = \frac{\eta_m}{m}   
    \end{split}
\end{align}
where $\eta_0 \doteq0$, $m=\argmin_{k\in [M]} \{\nicefrac{\eta_k}{k}\}$, and $h$ can be interpreted as the adversarial risk of an REC constructed from an optimal set of classifiers $\calF'$ with the same individual risks as $\calF$. We make the following observations:

\textbf{Implications of lower bound}: The lower bound in \eqref{eq:thm2-bound} provides us with a \emph{fundamental limit} on the adversarial risk of RECs viz., it is \emph{impossible} for any REC constructed from $M$ classifiers with sorted risks $\{\eta_i\}_{i=1}^M$ to achieve an adversarial risk smaller than $\min_{k\in [M]} \{\nicefrac{\eta_k}{k}\} = \nicefrac{\eta_m}{m}$. This limit is not always achievable and generalizes the one in \eqref{eq:thm1-limit} which holds for $M=2$. Theorem~\ref{thm:bounds} states that \emph{if} the limit is achievable then the corresponding optimal sampling probability $\bm{\alpha}^*=\tp{\left[\frac{1}{m}\ ...\ \frac{1}{m}\ 0\ ...\ 0\right]}$. Note that this \underline{does not} imply that the optimal sampling probability is always equiprobable sampling $\forall \calF$!

Additionally, the lower bound in \eqref{eq:thm2-bound} provides guidelines for robustifying individual classifiers in order for randomized ensembling to enhance the overall adversarial risk. Given classifiers $f_1, ..., f_m$ obtained via any sequential ensemble training algorithm, a good rule of thumb for the classifier obtained via the training iteration $m+1$ is to have:
\begin{equation} \label{eq:thumb-rule}
    \eta_m \leq \eta_{m+1} \leq \left(1+\frac{1}{m}\right)\eta_m
\end{equation}
Note that only for $m=1$ does  \eqref{eq:thumb-rule} become a \emph{necessary} condition: If $\eta_2 > 2\eta_1$, then $f_1$ will always achieve better risk than an REC of $f_1$ and $f_2$.
If a training method generates classifiers $f_1, ..., f_M$ with risks: $\eta_1 <1$ and $\eta_i = 1$ $\forall i \in \{2,...,M\}$, i.e., only the first classifier is somewhat robust and the remaining $M-1$ classifiers are compromised (such as BAT), the lower bound in \eqref{eq:thm2-bound} reduces to:
\begin{equation}\label{eq:bat-bound}
    \eta(\bm{\alpha}) \geq \min\left\{\eta_1, \frac{1}{M}\right\}
\end{equation}
implying the \emph{necessary} condition $M \geq \lceil\eta_1^{-1}\rceil$ for RECs constructed from $\calF$ to achieve better risk than $f_1$. Note: the fact that this condition is violated by Pinot et al. \yrcite{pinot2020randomization} hints to the existence of strong attacks that can break it \citep{mrboost,dbouk2022adversarial}.
%From \eqref{eq:bat-bound}, we establish a \emph{necessary} condition for RECs in this setting: In order for an REC constructed from $\calF$ to achieve better risk than $f_1$, we must have $M \geq \lceil\eta_1^{-1}\rceil$. For example, some of the robustness claims in BAT \citep{pinot2020randomization}] violate this necessary condition, which provides further theoretical evidence in support of \citep{dbouk2022adversarial}] \ns{not sure why we need to provide theoretical support for \citep{dbouk2022adversarial}]}.

\begin{algorithm}[t]
   \caption{The Optimal Sampling Probability (OSP) Algorithm for Randomized Ensembles}
   \label{alg:osp}
\begin{algorithmic}[1]
   \STATE {\bfseries Input:} classifiers $\calF=\{f_i\}_{i=1}^M$, perturbation set $\calS$, attack algorithm $\mathsf{attack}$, training set $\{\tuple{z}_j\}_{j=1}^n$, initial step-size $a>0$, and number of iterations $T\geq 1$.
   \STATE {\bfseries Output:} optimal sampling probability $\bm{\alpha}^*$.
   \STATE initialize $\bm{\alpha}^{(1)} \in \Delta_M$, $\eta_{\text{best}} \leftarrow 1$ 
   \CCOMMENT{we find that $\bm{\alpha}^{(1)}=\tp{\left[\frac{1}{M}\ ...\ \frac{1}{M}\right]}$ performs well }
    \FOR{$t \in \{1,...,T\}$}
        \STATE $\vc{g}\leftarrow \vc{0}$, $a_t \leftarrow \frac{a}{t}$  %\COMMENT{initialize $\vc{g}$ with vector of $M$ zeros}
        \FOR{$j \in \{1,...,n\}$}
            \STATE $\bm{\delta}_j \leftarrow \mathsf{attack}\left(\calF,\calS,\bm{\alpha}^{(t)},\tuple{z}_j\right)$   %\COMMENT{adversarial perturbation for $\tuple{z}_j$}
            \STATE $\forall i\in[M]$: $g_i \leftarrow g_i + \identityf{f_i(\vc{x}_j+\bm{\delta}_j)\neq y_j}$
        \ENDFOR
        \STATE $\vc{g} \leftarrow \frac{1}{n}\vc{g}$ \COMMENT{ sub-gradient of $\eta(\bm{\alpha}^{(t)})$}
        \STATE $\eta^{(t)}\leftarrow \tp{\vc{g}}\bm{\alpha}^{(t)}$ \COMMENT{$\eta(\bm{\alpha}^{(t)})$}
        \SHORTIF{$\eta^{(t)} \leq \eta_{\text{best}}$}{$t_{\text{best}} \leftarrow t$, $\eta_{\text{best}} \leftarrow \eta^{(t)}$}
        \CCOMMENT{projection-update step}
        \STATE $\bm{\alpha}^{(t+1)} \leftarrow \proj{\Delta_M}{\bm{\alpha}^{(t)} - a_{t} \vc{g}}$ 
   \ENDFOR
   \STATE {\textbf{return}} $\bm{\alpha}^{(t_{\text{best}})}$
\end{algorithmic}
\end{algorithm}

\subsection{Optimal Sampling} \label{ssec:sampling}

%While Section~\ref{ssec:two} provides a simple and elegant solution for finding the optimal sampling probability for $M\leq3$ classifiers, the need for an efficient solution for the general case $M> 3$ remains. 
In this section, we leverage Proposition~\ref{prop:eta} to extend the results in Section~\ref{ssec:two} to provide a theoretically optimal and efficient solution for computing the optimal sampling probability (OSP) algorithm (Algorithm~\ref{alg:osp}) for $M>3$.

In practice, we do not know the true data distribution $\calD$. Instead, we are provided a training set $\tuple{z}_1, ..., \tuple{z}_n$, assumed to be sampled i.i.d. from $\calD$. Given the training set, and a fixed collection of classifiers $\calF$, we wish to find the optimal sampling probability:
\begin{align}
    \begin{split}
    \label{eq:min-alpha-opt}
    \argmin_{\bm{\alpha} \in \Delta_M} \hat{\eta}(\bm{\alpha})= \argmin_{\bm{\alpha} \in \Delta_M} \frac{1}{n} \sum_{j=1}^n \left(\max_{\bm{\delta}\in\calS}{ r(\tuple{z}_j,\bm{\delta},\bm{\alpha})}\right)
    \end{split}
\end{align}
where $r$ is the per-sample risk from \eqref{eq:per-sample-risk}. Note that the empirical adversarial risk $\hat{\eta}$ is also piece-wise linear and convex in $\bm{\alpha}$, and hence all our theoretical results apply naturally. In order to numerically solve \eqref{eq:min-alpha-opt}, we first require access to an adversarial attack oracle ($\mathsf{attack}$) for RECs that solves the internal maximization $\forall \calS, \calF, \tuple{z},$ and $\bm{\alpha}$.

%\begin{equation} \label{eq:attack}
%    \mathsf{a}\left(\calF,\calS,\bm{\alpha},\tuple{z}\right) = \argmax_{\bm{\delta}\in\calS}{ \sum_{i=1}^M \alpha_i \identityf{f_i(\vc{x}+\bm{\delta})\neq y}}
%\end{equation}
Using the oracle $\mathsf{attack}$, Algorithm~\ref{alg:osp} updates its solution iteratively given the adversarial error-rate of each classifier over the training set. The projection operator $\Pi_{\Delta_M}$ in Line (15) of Algorithm~\ref{alg:osp} ensures that the solution is a valid p.m.f.. Wang \& Carreira-Perpinan \yrcite{wang2013projection} provide a simple and exact method for computing $\Pi_{\Delta_M}$. Finally, we state the following result on the optimality of OSP: 
\begin{theorem}\label{thm:osp} The OSP algorithm output $\bm{\alpha}_{T}$ satisfies:
\begin{align}
    \begin{split}
     0 &\leq \hat{\eta}(\bm{\alpha}_{T}) - \hat{\eta}(\bm{\alpha}^*) \\
     &\leq \frac{\pnorm{\bm{\alpha}^{(1)}-\bm{\alpha}^*}{2}^2 + M a^2\sum_{t=1}^T t^{-2} }{2 a\sum_{t=1}^T t^{-1}} \xrightarrow[T \to \infty]{} 0       
    \end{split}
\end{align}
for all initial conditions $\bm{\alpha}^{(1)} \in \Delta_M$, $a>0$, where $\bm{\alpha}^*$ is a global minimum.
\end{theorem}

Theorem~\ref{thm:osp} follows from a direct application of the classic convergence result of the projected sub-gradient method for constrained convex minimization \cite{shor2012minimization}. The optimality of OSP relies on the existence of an attack oracle which may not always exist. However, attack algorithms such as ARC \cite{dbouk2022adversarial} were found to yield good results in the common setting of differentiable classifiers and $\ell_p$-restricted adversaries.

\begin{algorithm}[t]
   \caption{The Boosting Algorithm for Robust Randomized Ensembles (BARRE)}
   \label{alg:BARRE}
\begin{algorithmic}[1]
\STATE {\bfseries Input:} Number of classifiers $M$, perturbation set $\calS$, training set $\{\tuple{z}_j\}_{j=1}^n$, learning rate $\rho$, mini-batch size $B$, number of epochs $E$, OSP frequency $E_o$, OSP number of iterations $T_o$.
   \STATE {\bfseries Output:} Robust randomized ensemble classifier $(\calF, \bm{\alpha})$
   \STATE initialize $\bm{\theta}_0 \in \Theta$, $\calF \leftarrow \varnothing$ %\COMMENT{randomly initialize the parameters of the first classifier}
    \FOR{$m \in \{1,...,M\}$}
        \STATE $\bm{\theta}_m \leftarrow \bm{\theta}_{m-1}$, $\calF \leftarrow \calF \cup \{f(\cdot|\bm{\theta}_m)\}$, $\bm{\alpha} \leftarrow \tp{\left[\frac{1}{m}\ ...\ \frac{1}{m}\right]}$  %\COMMENT{inherit the last parameters}
        \FOR{$e \in \{1,...,E\}$}
            \FOR{mini-batch $\{\tuple{z}_b\}_b^B$}
                \STATE compute $\forall b\in [B]$: $\bm{\delta}_b \leftarrow \mathsf{attack}\left(\calF,\calS,\bm{\alpha},\tuple{z}_b\right)$   %\COMMENT{adversarial perturbation at $\tuple{z}_b$}
                \STATE update $\bm{\theta}_{m}$ via SGD: 
                \begin{equation*}
                    \bm{\theta}_m \leftarrow \bm{\theta}_m - \frac{\rho}{B}\sum_{b=1}^B \nabla_{\bm{\theta}_m} l\left(\tilde{f}(\vc{x}_b + \bm{\delta}_b|\bm{\theta}_m),y_b\right) 
                \end{equation*}
            \ENDFOR
            \CCOMMENT{update $\bm{\alpha}$ every $E_o$ epochs}
            \IF{$e \mod E_o = 0$}
            \STATE {$\bm{\alpha} \leftarrow \mathsf{OSP} ( \calF, \calS, \{z_j\}_{j=1}^n, T_o)$} 
            \ENDIF
        \ENDFOR
        %\STATE $\bm{\alpha} \leftarrow \tp{\left((1-\frac{1}{m})\bm{\alpha} | \frac{1}{m}\right)}$  
   \ENDFOR
   \STATE {\textbf{return}} $\calF, \bm{\alpha}$
\end{algorithmic}
\end{algorithm}

\section{A Robust Boosting Algorithm for Randomized Ensembles}

In this section, we leverage our theoretical results in Section~\ref{sec:theory} as we explore designing robust RECs in practice.  Specifically, we propose BARRE: a unified \textbf{B}oosting \textbf{A}lgorithm for \textbf{R}obust \textbf{R}andomized \textbf{E}nsembles described in Algorithm~\ref{alg:BARRE}. Given a dataset $\{\tuple{z}_j\}_{j=1}^n$ and an REC attack algorithm $\mathsf{attack}$, BARRE iteratively trains a set of parametric classifiers $f(\cdot|\bm{\theta}_1)$, ..., $f(\cdot|\bm{\theta}_M)$ such that the adversarial risk of the corresponding REC is minimized. The first iteration of BARRE reduces to standard AT \citep{madry2018towards}. Doing so typically guarantees that the first classifier achieves the lowest adversarial risk and $\eta(\bm{\alpha}^*) \leq \eta_1$, i.e., Theorem~\ref{thm:osp} ensures the REC is \emph{no worse} than single model AT.

In each iteration $m\geq 2$, BARRE initializes the $m$-th classifier $f(\cdot|\bm{\theta}_m)$ with $\bm{\theta}_m = \bm{\theta}_{m-1}$. The training procedure alternates between updating the parameters $\bm{\theta}_m$ via SGD using adversarial samples of the current REC and solving for the optimal sampling probability $\bm{\alpha}^* \in \Delta_m$ via OSP. Including $f(\cdot|\bm{\theta}_m)$ in the attack (Line (8)) is crucial, as it ensures that the robustness of $f(\cdot|\bm{\theta}_m)$ is not completely compromised, thereby improving the bounds in Theorem~\ref{thm:bounds}. Note that for iterations $m\leq 3$, we replace the OSP procedure in Line (12) with a simplified search over a finite set of candidate solutions (see Section~\ref{ssec:two}). 

Furthermore, the rationale behind the sequence of steps in BARRE can be better understood using Theorem~\ref{thm:two-classifiers} (for the case of $M=2$). Theorem~\ref{thm:two-classifiers} states that the optimal REC adversarial risk would be $\eta(\bm{\alpha}^*)=\frac{1}{2}\left(\eta_1+\eta_2-\prob{\tuple{z}\in \calR_1}\right)$ (assuming \eqref{eq:thm1-cond} is met), therefore it is equally important to minimize both $\eta$'s and maximize $\prob{\tuple{z}\in \calR_1}$. BARRE does so by initially adversarially training a robust classifier $f_1$ (minimizing $\eta_1$), then training $f_2$ (initialized from $f_1$ to minimizes $\eta_2$) on the adversarial examples of the REC of $f_1$ and $f_2$. Doing so increases $\prob{\tuple{z}\in \calR_1}$ while maintaining $\eta_2$ as small as possible.

\begin{table}[hp]

\centering
\caption{Comparing BARRE with other methods in constructing robust RECs across network architectures and datasets. All methods incur the complexity of a single classifier. Robust accuracy is measured against an $\ell_\infty$ norm-bounded adversary using ARC with $\epsilon=\nicefrac{8}{255}$.} \label{tab:barre-vs-others}
%\vskip 0.05in

\vskip 0.15in
\begin{sc}
\resizebox{1\columnwidth}{!}{%
\begin{tabular}{l l c  c r  r r }

\toprule
\multirow{2}{*}{Network}   & \multirow{2}{*}{Method}  & \multirow{2}{*}{Size $M$}  & \multicolumn{2}{c}{CIFAR-10} & \multicolumn{2}{c}{CIFAR-100} \\
  &   & & $A_{\normalfont\text{nat}}$ [\%]& $A_{\normalfont\text{rob}}$ [\%] & $A_{\normalfont\text{nat}}$ [\%] & $A_{\normalfont\text{rob}}$ [\%]\\
\midrule
\multirow{4}[2]{*}{\shortstack{ResNet-20 \\ (81 MFLOPs)}}   &  AT & $M=1$ & $73.18$ & $41.99$ & $38.34$ & $17.69$\\
   \cmidrule(lr{1em}){2-7}
   & IAT & $M=5$ & $73.90$ & $45.77$ & $38.57$ & $19.65$\\
   & MRBoost-R & $M=5$ & $75.89$ & $46.66$ & $41.69$ & $21.04$\\
   %&  &\ \ \ \ \ + OSP& $M=5$ & same & same & z & z\\
   & BARRE & $M=5$ & \bfu{$76.28$} &	\bfu{$47.35$} & \bfu{$41.86$} & \bfu{$21.11$}\\
\midrule
\multirow{5}[4]{*}{\shortstack{MobileNetV1 \\ (312 MFLOPs)}}  &  AT & $M=1$ & $79.01$ & $46.22$ & $51.87$ & $23.45$\\
   \cmidrule(lr{1em}){2-7}
   &  IAT & $M=5$ & $78.89$ & $49.57$ & $51.41$	& $25.74$\\
   &  MRBoost-R$^\dagger$ & $M=5$ & $76.70$	 & $48.05$ & $50.14$ & $24.76$\\
   & MRBoost-R & $M=5$ & $78.65$ & $48.91$ & \bfu{$52.96$} & $25.95$\\
   &  BARRE & $M=5$ & \bfu{$79.55$} & \bfu{$49.91$} & $52.95$ & \bfu{$27.53$}\\
\midrule
\multirow{5}[4]{*}{\shortstack{ResNet-18\\ (1.1 GFLOPs)}}  &  AT & $M=1$ & $80.96$ & $48.72$ & $53.85$ & $24.15$\\
  \cmidrule(lr{1em}){2-7}
   & IAT & $M=4$ & $80.99$ & $51.43$ & $54.30$ & $26.73$\\
   & MRBoost-R$^\dagger$ & $M=4$ &$83.13$ & $51.82$ & $51.06$ & $24.04$\\
   & MRBoost-R & $M=4$ &$83.13$ & $51.82$ & $52.04$ & $25.65$\\
   & BARRE & $M=4$ & \bfu{$83.54$} & \bfu{$52.13$} & \bfu{$54.63$} & \bfu{$26.93$}\\
\bottomrule
\multicolumn{6}{l}{$\dagger$ \normalfont{result obtained assuming equiprobable sampling instead of using OSP}}

\end{tabular}
}
\end{sc}
\end{table}

\begin{table*}[t]
\centering
\caption{Comparison between BARRE and MRBoost across different network architectures and ensemble sizes on CIFAR-10. Robust accuracy is measured against an $\ell_\infty$ norm-bounded adversary using ARC with $\epsilon=\nicefrac{8}{255}$.} \label{tab:barre-vs-mrboost-cifar10}
\vskip 0.15in
\begin{sc}
\resizebox{1.8\columnwidth}{!}{%
\begin{tabular}{l  l| r  r r | r  r r | r r r | r r r}

\toprule
\multirow{2}{*}{Network}  &  \multirow{2}{*}{Method}& \multicolumn{3}{c|}{$M=1$} & \multicolumn{3}{c|}{$M=2$}  & \multicolumn{3}{c|}{$M=3$} & \multicolumn{3}{c}{$M=4$}
\\
  &   & $A_{\normalfont\text{nat}}$ & $A_{\normalfont\text{rob}}$   & FLOPs & $A_{\normalfont\text{nat}}$ & $A_{\normalfont\text{rob}}$   & FLOPs & $A_{\normalfont\text{nat}}$  & $A_{\normalfont\text{rob}}$ &  FLOPs & $A_{\normalfont\text{nat}}$ & $A_{\normalfont\text{rob}}$ &  FLOPs \\
  \midrule
  \multirow{2}{*}{ResNet-20} & MRBoost & \multirow{2}{*}{$73.18$}& \multirow{2}{*}{$41.99$}& \multirow{2}{*}{81 M}&$75.22$ & $44.68$ & 162 M & $76.13$ & $46.09$ & 243 M & $76.96$ & $46.34$ & 324 M  \\ % & $77.41$ & $46.89$ & 405 M
    & BARRE &  &  &  &$74.63$ & $44.38$ &  81 M & $75.55$ & $45.41$ &  81 M& $75.95$ & $46.44$ &  81 M \\ %$76.28$ & $47.13$ & 81 M
  \cmidrule(lr{1em}){2-14}
  \multirow{2}{*}{MobileNetV1} & MRBoost & \multirow{2}{*}{$79.01$}& \multirow{2}{*}{$46.22$}& \multirow{2}{*}{312 M}& $80.19$ & $48.58$ & 624 M & $79.79$ & $49.39$ & 936 M & $80.14$ & $49.36$ & 1.2 B  \\ %$79.33$ & $49.48$ & 1.6 B 
    & BARRE &  &  &  & $79.58$ & $48.32$&  312 M & $79.53$ & $48.75$ &  312 M& $79.54$ & $49.38$ &  312 M\\
  \cmidrule(lr{1em}){2-14}
  \multirow{2}{*}{ResNet-18} & MRBoost & \multirow{2}{*}{$80.96$}& \multirow{2}{*}{$48.7$}& \multirow{2}{*}{1.1 B}& $83.90$ & $50.72$ & 2.2 B & $85.07$ & $52.15$ & 3.3 B & $85.07$ & $52.15$ & 4.4 B   \\ %& $84.67$ & $52.85$ & 5.5 B
    & BARRE &  &  &  & $82.66$ & $50.51$ &  1.1 B & $83.40$ & $51.57$ & 1.1 B& $83.54$ & $52.13$ &  1.1 B \\
\bottomrule

\end{tabular}
}
\end{sc}
\end{table*}
\subsection{Experimental Results} \label{ssec:experiments}
%In this section, we validate the effectiveness of BARRE in constructing robust RECs. 

\textbf{Setup}. Per standard practice, we focus on defending against $\ell_\infty$ norm-bounded adversaries. We report results for three network architectures with different complexities: ResNet-20 \citep{he2016deep}, MobileNetV1 \citep{howard2017mobilenets}, and ResNet-18, across CIFAR-10 and CIFAR-100 datasets \citep{cifar10}. Computational complexity is measured via the number of floating-point operations (FLOPs) required per inference. The discrete nature of RECs allows us to compute the adversarial risk (or accuracy) \emph{exactly}. To ensure a fair comparison across different baselines, we use the same hyper-parameter settings detailed in Appendix~\ref{app:setup}. 

\textbf{Attack Algorithm}. For all our robust evaluations, we will adopt the state-of-the-art ARC algorithm \citep{dbouk2022adversarial} which can be used for both RECs and single models. Specifically, we shall use a slightly modified version that achieves better results in the equiprobable sampling setting (see Appendix~\ref{app:improved-arc}). For training with BARRE,  
%we find that using ARC (Line (8)) leads to severe overfitting, and instead 
we adopt adaptive PGD \citep{mrboost} for better generalization performance (see Appendix~\ref{app:arc-vs-apgd}).

\textbf{Results}. We first explore the efficacy of BARRE in constructing RECs that are robust against strong adversarial examples. Since there is an apparent lack of dedicated randomized ensemble defense methods in the literature, amplified further by the recent vulnerability of BAT \citep{pinot2020randomization}, we  establish baselines by constructing RECs from classifiers trained using MRBoost (denoted as MRBoost-R) and independent adversarial training (IAT). While MRBoost is dedicated to designing robust deterministic ensemble classifiers, it seems intuitive to investigate how well does the same ensemble performs when we randomly sample it. IAT, on the other hand, simply adversarially trains a set of classifiers using different random initialization. Thus, IAT does not enforce any explicit diversity within the ensemble, but maintains the highest individual model robustness. We use OSP (Algorithm~\ref{alg:osp}) to find the optimal sampling probability for each REC. All RECs share the same first classifier $f_1$, which is adversarially trained. Doing so ensures a fair comparison, and guarantees that none of the methods are worse than AT.

Table~\ref{tab:barre-vs-others} summarizes the performance of each method across network architectures and datasets. We note that all methods provide significant improvement in robustness compared to single model AT, indicating that RECs indeed provide increased robustness \emph{in practice} while maintaining compute complexity. Table~\ref{tab:barre-vs-others} provides evidence that the proposed BARRE algorithm outperforms both IAT and MRBoost. Interestingly, we find that MRBoost ensembles can be quite ill-suited for RECs. This can be seen for MobileNetV1, where the MRBoost REC achieves good performance only after completely disregarding the last classifier, i.e., the optimal sampling probability obtained was $\bm{\alpha}^{*}=\tp{[0.25\ 0.25\ 0.25\ 0.25\ 0]}$. This is due to the fact that MRBoost is meant for deterministic ensembles, and thus does not guarantee good performance in the randomized setting. In contrast, both IAT and BARRE-trained RECs utilize all members of the ensemble with non-zero probabilities.

%Recall that RECs hold an intrinsic complexity advantage over deterministic ensembles viz. REC's complexity remains unchanged with ensemble size $M$ whereas the complexity of deterministic ensembles increases linearly. Both methods have the same\footnote{ignoring the negligible memory overhead of storing $\bm{\alpha}$} memory footprint. 
We now compare the robustness and complexity of BARRE-trained RECs and MRBoost-trained deterministic ensembles. While both methods have the same\footnote{ignoring the negligible memory overhead of storing $\bm{\alpha}$} memory footprint, the computational complexity of RECs is $1/M$ of that of deterministic ensembles for ensemble size $M$. 
Table~\ref{tab:barre-vs-mrboost-cifar10} demonstrates that BARRE can successfully construct RECs that achieve competitive robustness (within $\sim 0.5\%$) compared to MRBoost-trained deterministic ensembles, across three different network architectures on CIFAR-10. The benefit of randomization can be seen for $M\geq 2$, as we obtain \emph{massive} $2-4\times$ savings in compute requirements.  These observations are further corroborated by CIFAR-100 experiments in Appendix~\ref{app:additional-results}.

%Table \ref{tab:barre-vs-mrboost-cifar10} demonstrates that BARRE can successfully construct RECs that achieve competitive robustness (within $\sim 0.5\%$) compared to MRBoost-trained deterministic ensembles, across three different network architectures on CIFAR-10. The benefit of randomization can be seen for $M\geq 2$, as we obtain \emph{massive} $2-4\times$ savings in compute requirements. Note that both methods have the same\footnote{ignoring the negligible memory overhead of storing $\bm{\alpha}$} memory footprint. These observations are further corroborated by CIFAR-100 experiments in Appendix~\ref{app:additional-results}.

%\textbf{BARRE vs. Other Methods}. Due to the lack of dedicated randomized ensemble defenses, we establish baselines by constructing RECs from both MRBoost and independently adversarially trained (IAT) models. We use OSP (Algorithm~\ref{alg:osp}) to find the optimal sampling probability for each REC. All RECs share the same first classifier $f_1$, which is adversarially trained. Doing so ensures a fair comparison, and guarantees all the methods cannot be worse than AT. Table~\ref{tab:barre-vs-others} provides strong evidence that BARRE outperforms both IAT and MRBoost for both CIFAR-10 and CIFAR-100 datasets. Interestingly, we find that MRBoost ensembles can be quite ill-suited for RECs. This can be seen for MobileNetV1, where the optimal sampling probability obtained was $\bm{\alpha}^{*}=\tp{[0.25\ 0.25\ 0.25\ 0.25\ 0]}$, i.e., the REC completely disregards the last classifier. In contrast, BARRE-trained RECs utilize all members of the ensemble.

\section{Discussion}
We have demonstrated both theoretically and empirically that robust randomized ensemble classifiers (RECs) are realizable. Theoretically, we derive the robustness limits of RECs, necessary and sufficient conditions for them to be useful, and efficient methods for finding the optimal sampling probability. Guided by theory, we propose BARRE, a new boosting algorithm for constructing robust RECs and demonstrate its effectiveness at defending against strong $\ell_\infty$ norm-bounded adversaries. 

Despite the empirical effectiveness of BARRE, there is a decent gap between the theoretical limits of RECs and robustness achieved in practice, leading us to believe there is much room for improvement in terms of achievable robustness.

\section*{Acknowledgements}

This work was supported by the Center for the Co-Design of Cognitive Systems (CoCoSys) funded by the Semiconductor Research Corporation (SRC) and the Defense Advanced Research Projects Agency (DARPA), and SRC’s Artificial Intelligence Hardware (AIHW) program.

%This work was supported in part by the Center for Co-Design of Cognitive Systems (CoCoSys), sponsored by Semiconductor Research Corporation (SRC) and Defense Advanced Research Projects Agency (DARPA) under the JUMP 2.0 program.
%This work was supported by the Center for the Co-Design of Cognitive Systems (CoCoSys) and the Artificial Intelligence Hardware (AIHW) program funded by the Semiconductor Research Corporation (SRC) and the Defense Advanced Research Projects Agency (DARPA).

\bibliography{ref}

\begin{thebibliography}{47}
\providecommand{\natexlab}[1]{#1}
\providecommand{\url}[1]{\texttt{#1}}
\expandafter\ifx\csname urlstyle\endcsname\relax
  \providecommand{\doi}[1]{doi: #1}\else
  \providecommand{\doi}{doi: \begingroup \urlstyle{rm}\Url}\fi

\bibitem[Abernethy et~al.(2021)Abernethy, Awasthi, and
  Kale]{abernethy2021multiclass}
Abernethy, J., Awasthi, P., and Kale, S.
\newblock A multiclass boosting framework for achieving fast and provable
  adversarial robustness.
\newblock \emph{arXiv preprint arXiv:2103.01276}, 2021.

\bibitem[Athalye et~al.(2018)Athalye, Carlini, and
  Wagner]{athalye2018obfuscated}
Athalye, A., Carlini, N., and Wagner, D.
\newblock Obfuscated gradients give a false sense of security: Circumventing
  defenses to adversarial examples.
\newblock In \emph{International Conference on Machine Learning}, pp.\
  274--283. PMLR, 2018.

\bibitem[Biggio et~al.(2013)Biggio, Corona, Maiorca, Nelson, {\v{S}}rndi{\'c},
  Laskov, Giacinto, and Roli]{biggio2013evasion}
Biggio, B., Corona, I., Maiorca, D., Nelson, B., {\v{S}}rndi{\'c}, N., Laskov,
  P., Giacinto, G., and Roli, F.
\newblock Evasion attacks against machine learning at test time.
\newblock In \emph{Joint European conference on machine learning and knowledge
  discovery in databases}, pp.\  387--402. Springer, 2013.

\bibitem[Boyd et~al.(2004)Boyd, Boyd, and Vandenberghe]{boyd2004convex}
Boyd, S., Boyd, S.~P., and Vandenberghe, L.
\newblock \emph{Convex optimization}.
\newblock Cambridge university press, 2004.

\bibitem[Breiman(1996)]{breiman1996bagging}
Breiman, L.
\newblock Bagging predictors.
\newblock \emph{Machine learning}, 24\penalty0 (2):\penalty0 123--140, 1996.

\bibitem[Carbone et~al.(2020)Carbone, Wicker, Laurenti, Patane, Bortolussi, and
  Sanguinetti]{carbone2020robustness}
Carbone, G., Wicker, M., Laurenti, L., Patane, A., Bortolussi, L., and
  Sanguinetti, G.
\newblock Robustness of bayesian neural networks to gradient-based attacks.
\newblock \emph{Advances in Neural Information Processing Systems},
  33:\penalty0 15602--15613, 2020.

\bibitem[Carlini \& Wagner(2017)Carlini and Wagner]{carlini2017towards}
Carlini, N. and Wagner, D.
\newblock Towards evaluating the robustness of neural networks.
\newblock In \emph{2017 ieee symposium on security and privacy (sp)}, pp.\
  39--57. IEEE, 2017.

\bibitem[Cisse et~al.(2017)Cisse, Bojanowski, Grave, Dauphin, and
  Usunier]{cisse2017parseval}
Cisse, M., Bojanowski, P., Grave, E., Dauphin, Y., and Usunier, N.
\newblock Parseval networks: Improving robustness to adversarial examples.
\newblock In \emph{International Conference on Machine Learning}, pp.\
  854--863. PMLR, 2017.

\bibitem[Cohen et~al.(2019)Cohen, Rosenfeld, and Kolter]{cohen2019certified}
Cohen, J., Rosenfeld, E., and Kolter, Z.
\newblock Certified adversarial robustness via randomized smoothing.
\newblock In \emph{International Conference on Machine Learning}, pp.\
  1310--1320. PMLR, 2019.

\bibitem[Dbouk \& Shanbhag(2021)Dbouk and Shanbhag]{dbouk2021generalized}
Dbouk, H. and Shanbhag, N.
\newblock Generalized depthwise-separable convolutions for adversarially robust
  and efficient neural networks.
\newblock \emph{Advances in Neural Information Processing Systems}, 34, 2021.

\bibitem[Dbouk \& Shanbhag(2022)Dbouk and Shanbhag]{dbouk2022adversarial}
Dbouk, H. and Shanbhag, N.
\newblock Adversarial vulnerability of randomized ensembles.
\newblock In \emph{International Conference on Machine Learning}, pp.\
  4890--4917. PMLR, 2022.

\bibitem[Dhillon et~al.(2018)Dhillon, Azizzadenesheli, Lipton, Bernstein,
  Kossaifi, Khanna, and Anandkumar]{dhillon2018stochastic}
Dhillon, G.~S., Azizzadenesheli, K., Lipton, Z.~C., Bernstein, J.~D., Kossaifi,
  J., Khanna, A., and Anandkumar, A.
\newblock Stochastic activation pruning for robust adversarial defense.
\newblock In \emph{International Conference on Learning Representations}, 2018.

\bibitem[Dietterich(2000)]{dietterich2000ensemble}
Dietterich, T.~G.
\newblock Ensemble methods in machine learning.
\newblock In \emph{International workshop on multiple classifier systems}, pp.\
   1--15. Springer, 2000.

\bibitem[Duan et~al.(2020)Duan, Ma, Wang, Bailey, Qin, and
  Yang]{duan2020adversarial}
Duan, R., Ma, X., Wang, Y., Bailey, J., Qin, A.~K., and Yang, Y.
\newblock Adversarial camouflage: Hiding physical-world attacks with natural
  styles.
\newblock In \emph{Proceedings of the IEEE/CVF conference on computer vision
  and pattern recognition}, pp.\  1000--1008, 2020.

\bibitem[Freund \& Schapire(1997)Freund and Schapire]{freund1997decision}
Freund, Y. and Schapire, R.~E.
\newblock A decision-theoretic generalization of on-line learning and an
  application to boosting.
\newblock \emph{Journal of computer and system sciences}, 55\penalty0
  (1):\penalty0 119--139, 1997.

\bibitem[Goodfellow et~al.(2014)Goodfellow, Shlens, and
  Szegedy]{goodfellow2014explaining}
Goodfellow, I.~J., Shlens, J., and Szegedy, C.
\newblock Explaining and harnessing adversarial examples.
\newblock \emph{arXiv preprint arXiv:1412.6572}, 2014.

\bibitem[Guo et~al.(2020)Guo, Yang, Xu, Liu, and Lin]{NAS}
Guo, M., Yang, Y., Xu, R., Liu, Z., and Lin, D.
\newblock When {NAS} meets robustness: In search of robust architectures
  against adversarial attacks.
\newblock In \emph{Proceedings of the IEEE/CVF Conference on Computer Vision
  and Pattern Recognition}, pp.\  631--640, 2020.

\bibitem[He et~al.(2016)He, Zhang, Ren, and Sun]{he2016deep}
He, K., Zhang, X., Ren, S., and Sun, J.
\newblock Deep residual learning for image recognition.
\newblock In \emph{Proceedings of the IEEE conference on computer vision and
  pattern recognition}, pp.\  770--778, 2016.

\bibitem[Howard et~al.(2017)Howard, Zhu, Chen, Kalenichenko, Wang, Weyand,
  Andreetto, and Adam]{howard2017mobilenets}
Howard, A.~G., Zhu, M., Chen, B., Kalenichenko, D., Wang, W., Weyand, T.,
  Andreetto, M., and Adam, H.
\newblock Mobilenets: Efficient convolutional neural networks for mobile vision
  applications.
\newblock \emph{arXiv preprint arXiv:1704.04861}, 2017.

\bibitem[Kariyappa \& Qureshi(2019)Kariyappa and
  Qureshi]{kariyappa2019improvingGAL}
Kariyappa, S. and Qureshi, M.~K.
\newblock Improving adversarial robustness of ensembles with diversity
  training.
\newblock \emph{arXiv preprint arXiv:1901.09981}, 2019.

\bibitem[Katz et~al.(2017)Katz, Barrett, Dill, Julian, and
  Kochenderfer]{katz2017reluplex}
Katz, G., Barrett, C., Dill, D.~L., Julian, K., and Kochenderfer, M.~J.
\newblock Reluplex: An efficient {SMT} solver for verifying deep neural
  networks.
\newblock In \emph{International conference on computer aided verification},
  pp.\  97--117. Springer, 2017.

\bibitem[Krizhevsky et~al.(2009)Krizhevsky, Hinton, et~al.]{cifar10}
Krizhevsky, A., Hinton, G., et~al.
\newblock Learning multiple layers of features from tiny images.
\newblock Technical report, Citeseer, 2009.

\bibitem[Liu et~al.(2018)Liu, Yang, Liu, Song, Li, and Chen]{liu2018dpatch}
Liu, X., Yang, H., Liu, Z., Song, L., Li, H., and Chen, Y.
\newblock Dpatch: An adversarial patch attack on object detectors.
\newblock \emph{arXiv preprint arXiv:1806.02299}, 2018.

\bibitem[Madry et~al.(2018)Madry, Makelov, Schmidt, Tsipras, and
  Vladu]{madry2018towards}
Madry, A., Makelov, A., Schmidt, L., Tsipras, D., and Vladu, A.
\newblock Towards deep learning models resistant to adversarial attacks.
\newblock In \emph{International Conference on Learning Representations}, 2018.
\newblock URL \url{https://openreview.net/forum?id=rJzIBfZAb}.

\bibitem[Neal(2012)]{neal2012bayesian}
Neal, R.~M.
\newblock \emph{Bayesian learning for neural networks}, volume 118.
\newblock Springer Science \& Business Media, 2012.

\bibitem[Pang et~al.(2019)Pang, Xu, Du, Chen, and Zhu]{pang2019improvingADP}
Pang, T., Xu, K., Du, C., Chen, N., and Zhu, J.
\newblock Improving adversarial robustness via promoting ensemble diversity.
\newblock In \emph{International Conference on Machine Learning}, pp.\
  4970--4979. PMLR, 2019.

\bibitem[Papernot et~al.(2016)Papernot, McDaniel, Wu, Jha, and
  Swami]{papernot2016distillation}
Papernot, N., McDaniel, P., Wu, X., Jha, S., and Swami, A.
\newblock Distillation as a defense to adversarial perturbations against deep
  neural networks.
\newblock In \emph{2016 IEEE symposium on security and privacy (SP)}, pp.\
  582--597. IEEE, 2016.

\bibitem[Pinot et~al.(2020)Pinot, Ettedgui, Rizk, Chevaleyre, and
  Atif]{pinot2020randomization}
Pinot, R., Ettedgui, R., Rizk, G., Chevaleyre, Y., and Atif, J.
\newblock Randomization matters how to defend against strong adversarial
  attacks.
\newblock In \emph{International Conference on Machine Learning}, pp.\
  7717--7727. PMLR, 2020.

\bibitem[Raghunathan et~al.(2018)Raghunathan, Steinhardt, and
  Liang]{raghunathan2018certified}
Raghunathan, A., Steinhardt, J., and Liang, P.
\newblock Certified defenses against adversarial examples.
\newblock In \emph{International Conference on Learning Representations}, 2018.

\bibitem[Rice et~al.(2020)Rice, Wong, and Kolter]{rice2020overfitting}
Rice, L., Wong, E., and Kolter, Z.
\newblock Overfitting in adversarially robust deep learning.
\newblock In \emph{International Conference on Machine Learning}, pp.\
  8093--8104. PMLR, 2020.

\bibitem[Sehwag et~al.(2020)Sehwag, Wang, Mittal, and Jana]{sehwag2020hydra}
Sehwag, V., Wang, S., Mittal, P., and Jana, S.
\newblock {HYDRA}: Pruning adversarially robust neural networks.
\newblock \emph{Advances in Neural Information Processing Systems (NeurIPS)},
  7, 2020.

\bibitem[Sen et~al.(2019)Sen, Ravindran, and Raghunathan]{sen2019empir}
Sen, S., Ravindran, B., and Raghunathan, A.
\newblock {EMPIR}: Ensembles of mixed precision deep networks for increased
  robustness against adversarial attacks.
\newblock In \emph{International Conference on Learning Representations}, 2019.

\bibitem[Shor(2012)]{shor2012minimization}
Shor, N.~Z.
\newblock \emph{Minimization methods for non-differentiable functions},
  volume~3.
\newblock Springer Science \& Business Media, 2012.

\bibitem[Szegedy et~al.(2013)Szegedy, Zaremba, Sutskever, Bruna, Erhan,
  Goodfellow, and Fergus]{szegedy2013intriguing}
Szegedy, C., Zaremba, W., Sutskever, I., Bruna, J., Erhan, D., Goodfellow, I.,
  and Fergus, R.
\newblock Intriguing properties of neural networks.
\newblock \emph{arXiv preprint arXiv:1312.6199}, 2013.

\bibitem[Tjeng et~al.(2018)Tjeng, Xiao, and Tedrake]{tjeng2018evaluating}
Tjeng, V., Xiao, K.~Y., and Tedrake, R.
\newblock Evaluating robustness of neural networks with mixed integer
  programming.
\newblock In \emph{International Conference on Learning Representations}, 2018.

\bibitem[Tram{\`e}r et~al.(2020)Tram{\`e}r, Carlini, Brendel, and
  Madry]{tramer2020adaptive}
Tram{\`e}r, F., Carlini, N., Brendel, W., and Madry, A.
\newblock On adaptive attacks to adversarial example defenses.
\newblock \emph{Advances in Neural Information Processing Systems}, 33, 2020.

\bibitem[Tramèr et~al.(2018)Tramèr, Kurakin, Papernot, Goodfellow, Boneh, and
  McDaniel]{tramer2018ensemble}
Tramèr, F., Kurakin, A., Papernot, N., Goodfellow, I., Boneh, D., and
  McDaniel, P.
\newblock Ensemble adversarial training: Attacks and defenses.
\newblock In \emph{International Conference on Learning Representations}, 2018.
\newblock URL \url{https://openreview.net/forum?id=rkZvSe-RZ}.

\bibitem[Vorobeychik \& Li(2014)Vorobeychik and Li]{vorobeychik2014optimal}
Vorobeychik, Y. and Li, B.
\newblock Optimal randomized classification in adversarial settings.
\newblock In \emph{AAMAS}, pp.\  485--492, 2014.

\bibitem[Wang \& Carreira-Perpin{\'a}n(2013)Wang and
  Carreira-Perpin{\'a}n]{wang2013projection}
Wang, W. and Carreira-Perpin{\'a}n, M.~A.
\newblock Projection onto the probability simplex: An efficient algorithm with
  a simple proof, and an application.
\newblock \emph{arXiv preprint arXiv:1309.1541}, 2013.

\bibitem[Xiao et~al.(2018)Xiao, Tjeng, Shafiullah, and Madry]{xiao2018training}
Xiao, K.~Y., Tjeng, V., Shafiullah, N. M.~M., and Madry, A.
\newblock Training for faster adversarial robustness verification via inducing
  {ReLU} stability.
\newblock In \emph{International Conference on Learning Representations}, 2018.

\bibitem[Xie et~al.(2018)Xie, Wang, Zhang, Ren, and Yuille]{xie2018mitigating}
Xie, C., Wang, J., Zhang, Z., Ren, Z., and Yuille, A.
\newblock Mitigating adversarial effects through randomization.
\newblock In \emph{International Conference on Learning Representations}, 2018.

\bibitem[Yang et~al.(2020{\natexlab{a}})Yang, Duan, Hu, Salman, Razenshteyn,
  and Li]{yang2020randomized}
Yang, G., Duan, T., Hu, J.~E., Salman, H., Razenshteyn, I., and Li, J.
\newblock Randomized smoothing of all shapes and sizes.
\newblock In \emph{International Conference on Machine Learning}, pp.\
  10693--10705. PMLR, 2020{\natexlab{a}}.

\bibitem[Yang et~al.(2020{\natexlab{b}})Yang, Zhang, Dong, Inkawhich, Gardner,
  Touchet, Wilkes, Berry, and Li]{yang2020dverge}
Yang, H., Zhang, J., Dong, H., Inkawhich, N., Gardner, A., Touchet, A., Wilkes,
  W., Berry, H., and Li, H.
\newblock {DVERGE}: Diversifying vulnerabilities for enhanced robust generation
  of ensembles.
\newblock \emph{Advances in Neural Information Processing Systems}, 33,
  2020{\natexlab{b}}.

\bibitem[Yang et~al.(2019)Yang, Zhang, Katabi, and Xu]{yang2019me}
Yang, Y., Zhang, G., Katabi, D., and Xu, Z.
\newblock Me-net: Towards effective adversarial robustness with matrix
  estimation.
\newblock In \emph{International Conference on Machine Learning}, pp.\
  7025--7034. PMLR, 2019.

\bibitem[Yang et~al.(2021)Yang, Li, Xu, Zuo, Chen, Zhou, Rubinstein, Zhang, and
  Li]{yang2021trs}
Yang, Z., Li, L., Xu, X., Zuo, S., Chen, Q., Zhou, P., Rubinstein, B., Zhang,
  C., and Li, B.
\newblock {TRS}: Transferability reduced ensemble via promoting gradient
  diversity and model smoothness.
\newblock \emph{Advances in Neural Information Processing Systems}, 34, 2021.

\bibitem[Zhang et~al.(2022)Zhang, Zhang, Courville, Bengio, Ravikumar, and
  Suggala]{mrboost}
Zhang, D., Zhang, H., Courville, A., Bengio, Y., Ravikumar, P., and Suggala,
  A.~S.
\newblock Building robust ensembles via margin boosting.
\newblock In \emph{Proceedings of the 39th International Conference on Machine
  Learning}, volume 162, pp.\  26669--26692. PMLR, 17--23 Jul 2022.

\bibitem[Zhang et~al.(2019)Zhang, Yu, Jiao, Xing, El~Ghaoui, and
  Jordan]{trades}
Zhang, H., Yu, Y., Jiao, J., Xing, E., El~Ghaoui, L., and Jordan, M.
\newblock Theoretically principled trade-off between robustness and accuracy.
\newblock In \emph{International Conference on Machine Learning}, pp.\
  7472--7482. PMLR, 2019.

\end{thebibliography}
\bibliographystyle{icml2023}

%%%%%%%%%%%%%%%%%%%%%%%%%%%%%%%%%%%%%%%%%%%%%%%%%%%%%%%%%%%%%%%%%%%%%%%%%%%%%%%
%%%%%%%%%%%%%%%%%%%%%%%%%%%%%%%%%%%%%%%%%%%%%%%%%%%%%%%%%%%%%%%%%%%%%%%%%%%%%%%
% APPENDIX
%%%%%%%%%%%%%%%%%%%%%%%%%%%%%%%%%%%%%%%%%%%%%%%%%%%%%%%%%%%%%%%%%%%%%%%%%%%%%%%
%%%%%%%%%%%%%%%%%%%%%%%%%%%%%%%%%%%%%%%%%%%%%%%%%%%%%%%%%%%%%%%%%%%%%%%%%%%%%%%
\newpage
\appendix
\onecolumn

\section{Omitted Proofs and Derivations}\label{app:proofs}

\subsection{Proof of Proposition~\ref{prop:eta}}
We provide the proof of Proposition~\ref{prop:eta} (restated below):
\begin{proposition*}[Restated]
For any $\calF=\{f_i\}_{i=1}^M$, perturbation set $\calS \subset \reals^d$, and data distribution $\calD$, the adversarial risk $\eta$ is a piece-wise linear convex function $\forall \bm{\alpha} \in \Delta_M$. Specifically, $\exists K\in \naturals$ configurations $\calU_k\subseteq\{0,1\}^M$ $\forall k\in[K]$ and p.m.f. $\vc{p}\in \Delta_K$ such that:
\begin{equation}%\label{eq:eta-2}
    \eta(\bm{\alpha}) = \sum_{k=1}^K\left(p_k \cdot \max_{\vc{u}\in \calU_k}\left\{\tp{\vc{u}}\bm{\alpha}\right\}\right)
\end{equation}
\end{proposition*}
\begin{proof}

Consider having one data-point $\tuple{z}\in \reals^d \times [C]$, then for any $\bm{\delta} \in \calS$ and $\bm{\alpha}\in \Delta_M$ we have:
\begin{equation}
    r\left(\tuple{z},\bm{\delta},\bm{\alpha}\right)= \sum_{i=1}^M \alpha_i \identityf{f_i(\vc{x}+\bm{\delta})\neq y}  = \tp{\vc{u}}\bm{\alpha}
\end{equation}
where $\vc{u}\in\{0,1\}^M$ such that $u_i=1$ if and only if $\bm{\delta}$ is adversarial to $f_i$ at $\tuple{z}$. Since $\vc{u}$ is independent of $\bm{\alpha}$, we thus obtain a many-to-one mapping from $\bm{\delta}\in \calS$ to $\vc{u}\in\{0,1\}^M$. Therefore, for any $\bm{\alpha}$ and $\tuple{z}$, we can always decompose the perturbation set $\calS$, i.e., $\calS = \calG_1 \cup ... \cup \calG_n$, into $n\leq 2^M$ subsets, such that: $\forall \bm{\delta} \in \calG_j: r\left(\tuple{z},\bm{\delta},\bm{\alpha}\right) = \tp{\bm{\alpha}}\vc{u}_j$
for some binary vector $\vc{u}_j$ independent of $\bm{\alpha}$. Let $\calU=\{\vc{u}_j\}_{j=1}^{n}$ be the collection of these vectors, then we can write:
\begin{align}% \label{eq:singe-risk-equiv}
    \begin{split}
        \max_{\bm{\delta}\in\calS} r\left(\tuple{z},\bm{\delta},\bm{\alpha}\right) &= \max_{\bm{\delta}\in\calG_1 \cup ... \cup \calG_n} r\left(\tuple{z},\bm{\delta},\bm{\alpha}\right)=\max_{j \in [n]} \left\{ \max_{\bm{\delta}\in\calG_j} r\left(\tuple{z},\bm{\delta},\bm{\alpha}\right)\right\} = \max_{\vc{u}\in \calU} \left\{ \tp{\vc{u}}\bm{\alpha}\right\}
    \end{split}
\end{align}

The vectors $\{\vc{u}_j\}_{j=1}^n$ define a unique classifier and data-point configuration that is independent of the sampling probability. The function $\max_{\bm{\delta}} r$ is thus convex and piece-wise linear in $\bm{\alpha}$.

Partitioning the data-point space $\calR \subseteq \reals^d\times[C]$ into $K$ subsets $\calR=\calR_1 \cup ... \cup \calR_K$ such that all the data-points $\tuple{z}\in \calR_k$ share the same set \say{configuration} $\calU_k$, we obtain:
\begin{align}
    \begin{split}
        \eta(\bm{\alpha}) &= \means{\tuple{z}\sim \calD}{\max_{\bm{\delta}\in\calS}{ \sum_{i=1}^M \alpha_i \identityf{f_i(\vc{x}+\bm{\delta})\neq y}}} \\
        &= \int_{\tuple{z}\in\calR} p_z(\tuple{z}) \cdot \max_{\bm{\delta}\in\calS}{r\left(\tuple{z},\bm{\delta},\bm{\alpha}\right)} \,d\tuple{z}\\
        &= \sum_{k=1}^K \int_{\tuple{z}\in\calR_k} p_z(\tuple{z}) \cdot \max_{\bm{\delta}\in\calS}{r\left(\tuple{z},\bm{\delta},\bm{\alpha}\right)} \,d\tuple{z} \\
        &= \sum_{k=1}^K \int_{\tuple{z}\in\calR_k} p_z(\tuple{z}) \cdot \left( \max_{\vc{u}\in \calU_k} \left\{ \tp{\vc{u}}\bm{\alpha}\right\}\right) \,d\tuple{z} \\
        &= \sum_{k=1}^K \left(\max_{\vc{u}\in \calU_k} \left\{ \tp{\vc{u}}\bm{\alpha}\right\} \cdot \int_{\tuple{z}\in\calR_k} p_z(\tuple{z})  \,d\tuple{z}\right) \\
        &= \sum_{k=1}^K  \left(p_k \cdot \max_{\vc{u}\in \calU_k} \left\{ \tp{\vc{u}}\bm{\alpha}\right\} \right)
    \end{split}
\end{align}
where the total size of the partition $K$ is finite (exponential in the size $M$) and $\vc{p}\in\Delta_K$ such that $p_k = \prob{\tuple{z}\in \calR_k}$. Finally, $\eta$ is convex and piece-wise linear in $\bm{\alpha}$ since the summation of convex and piece-wise linear functions  is also convex and piece-wise linear.
\end{proof}

\subsection{Proof of Theorem~\ref{thm:two-classifiers}}
First, we state and prove this useful lemma: 
\begin{lemma}
\label{lemma:opt}
Let $h: \reals \rightarrow \reals$ be a convex piece-wise linear, hence sub-differentiable, function of the form:
\begin{equation}\label{eqn:hx}
    h(x) = \max\{a_1x+b_1, a_2x+b_2\} + a_3x+b_3
\end{equation}
such that $a_1 < a_2$. We wish to minimize $h$ over $x\in [c,d]$ where $c\leq y \leq d$, and $y$ is the intersection point $\frac{b_2-b_1}{a_1-a_2}$. 

Then, the optimal value $x^*$ that minimizes $h(x)$ in \eqref{eqn:hx}, is given by
\begin{align*}
    x^* &= \begin{cases}
         y,\qquad\text{if }a_3 \in (-a_2,-a_1)\\
         c,\qquad\text{if }a_3 \geq-a_1\\
         d,\qquad\text{if }a_3 \leq-a_2
         \end{cases}
\end{align*}
Note: only in the first case is the solution unique.
\end{lemma}

\begin{proof}
From constrained convex optimization (\cite{boyd2004convex, shor2012minimization}), we know that $x^*$ is the minimizer of $h$ over $[c,d]$ if there exists a sub-gradient $g \in \partial h(x^*)$ such that:
\begin{equation}
    g\cdot(x-x^*)\geq 0 \ \ \ \forall x \in [c,d]
\end{equation}

For $x\neq y$, $h$ is differentiable with $\nabla h = a_3 + a_1$ (if $x<y$) or $\nabla h = a_3 + a_2$ (if $x>y$), and for $x=y$ the sub-differential is given by $\partial h(y) = \{a_3 + \beta a_1 + (1-\beta)a_2: \beta \in [0,1]\}$. 

If $a_3 \in (-a_2,-a_1)$, then $\exists \beta \in [0,1]$ such that $a_3 + \beta a_1 + (1-\beta)a_2 = 0$, and thus $ 0 \in \partial h(y)$, which is a sufficient condition for global minimization, thus $x^* = y$. Furthermore, $x^*=y$ is unique, since $\forall x \neq y$, we will have $\nabla h = a_1 + a_3 <0$ (if $x<y$) or $\nabla h = a_2 + a_3 >0$ (if $x>y$) which in both cases implies $\forall z\neq y$ $\exists x \in [c,d]$ such that $\nabla h(z) (x-z)<0$.

If $a_3 \notin (-a_2,-a_1)$, then either $a_3 \geq -a_1$ or $a_3 \leq -a_2$. If $a_3 \geq-a_1$, then $a_1+a_3 = \nabla h(c) \geq0$, which implies that: $(a_1+a_3)(x-c) \geq 0$ $\forall x \in [c,d]$, hence $x^* =c$. Otherwise if $a_3 \leq -a_2$, then $a_2+a_3 = \nabla h(d) \leq0$, which implies that: $(a_2+a_3)(x-d) \geq 0$ $\forall x \in [c,d]$, hence $x^* =d$.
\end{proof}

We now provide the proof of Theorem~\ref{thm:two-classifiers} (restated below):
\begin{theorem*}[Restated]
For any two classifiers $f_1$ and $f_2$ with individual adversarial risks $\eta_1$ and $\eta_2$, respectively, subject to a perturbation set $\calS \subset \reals^d$ and data distribution $\calD$, if:
\begin{equation}
    \prob{\tuple{z}\in\calR_1} > |\eta_1-\eta_2|
\end{equation}
where:
\begin{equation}
    \calR_1 = \{\tuple{z}\in\reals^d \times [C]: \calS_1(\tuple{z})\neq \varnothing, \calS_2(\tuple{z})\neq \varnothing, \calS_1(\tuple{z})\cap \calS_2(\tuple{z})=\varnothing\}
\end{equation}
then the optimum sampling probability $\bm{\alpha}^*=\tp{(\nicefrac{1}{2},\nicefrac{1}{2})}$ uniquely minimizes $\eta(\bm{\alpha})$ resulting in $\eta(\bm{\alpha}^*)=\frac{1}{2}\left(\eta_1+\eta_2-\prob{\tuple{z}\in \calR_1}\right)$. Otherwise,  $\bm{\alpha}^* \in \{\vc{e}_1, \vc{e}_2\}$ minimizes $\eta(\bm{\alpha})$, where $\vc{e}_i$s are the standard basis vectors of $\reals^2$.
\end{theorem*}
\begin{proof}
We know that, for $M=2$, the adversarial risk $
\eta$ can be re-written $\forall \bm{\alpha}\in \Delta_2$:
\begin{equation}
    \eta(\bm{\alpha}) = p_1 \cdot \max\{\alpha_1,\alpha_2\} + p_2 \cdot 1 + p_3 \cdot \alpha_1 + p_4 \cdot \alpha_2 + p_5 \cdot 0
\end{equation}
where $p_k = \prob{\tuple{z}\in \calR_k}$, and the regions $\{\calR_k\}_{k=1}^K$ partition the input space $\reals^d \times [C]$ as follows:
\begin{align}
    \begin{split}
    \calR_1 &= \{\tuple{z}\in\reals^d \times [C]: \calS_1(\tuple{z})\neq \varnothing, \calS_2(\tuple{z})\neq \varnothing, \calS_1(\tuple{z})\cap \calS_2(\tuple{z})=\varnothing\} \\
    \calR_2 &= \{\tuple{z}\in\reals^d \times [C]: \calS_1(\tuple{z})\cap \calS_2(\tuple{z})\neq\varnothing\} \\
    \calR_3 &= \{\tuple{z}\in\reals^d \times [C]: \calS_1(\tuple{z}) \neq \varnothing, \calS_2(\tuple{z})=\varnothing\} \\ 
    \calR_4 &= \{\tuple{z}\in\reals^d \times [C]: \calS_1(\tuple{z}) = \varnothing, \calS_2(\tuple{z})\neq\varnothing\} \\
    \calR_5 &= \{\tuple{z}\in\reals^d \times [C]: \calS_1(\tuple{z}) =  \calS_2(\tuple{z})=\varnothing\}
    \end{split}
\end{align}

Using $\alpha_1=1-\alpha_2=\alpha$, we have $\forall \alpha \in[0,1]$:
\begin{equation}
    \eta\left(\tp{\left(\alpha,1-\alpha\right)}\right) = h(\alpha) =  p_1 \cdot \max\{\alpha,1-\alpha\} + (p_3-p_4) \cdot \alpha + p_2+p_4
\end{equation}
where we wish to find $\alpha^*\in [0,1]$ that minimizes $h(\alpha)$.  Applying Lemma~\ref{lemma:opt} with:
\begin{equation}
    a_1 = -p_1,\ b_1 = p_1,\ a_2=p_1,\ b_2 = 0,\ a_3 = p_3-p_4,\ b_3 = p_2+p_4
\end{equation}
and utilizing $\eta_1 = \eta(\vc{e}_1) = p_1 + p_2 + p_3$ and $\eta_2 = \eta(\vc{e}_2) = p_1 + p_2 + p_4$, yields the main result.
\end{proof}

\subsection{Proof of Corollary~\ref{cor:limit}}
\label{app:proof-corr}
We provide the proof of Corollary~\ref{cor:limit} (restated below):

\begin{corollary*} For any two classifiers $f_1$ and $f_2$ with individual adversarial risks $\eta_1$ and $\eta_2$, respectively, perturbation set $\calS$, and data distribution $\calD$: 
\begin{equation}\label{eq:thm1-limit-r}
   \min_{\bm{\alpha}\in \Delta_2} \eta(\bm{\alpha}) = \eta(\bm{\alpha}^*) \geq \min\left\{\frac{1}{2} \max\{\eta_1, \eta_2\}, \min\{\eta_1,\eta_2\}\right\}.
\end{equation}
\end{corollary*}
\begin{proof}
From Theorem~\ref{thm:two-classifiers}, we have that:
\begin{equation}
    \eta(\bm{\alpha}^*) = \min\left\{ \frac{1}{2}\left(\eta_1+\eta_2-\prob{\tuple{z}\in \calR_1}\right), \min\{\eta_1,\eta_2\}\right\}
\end{equation}
Using the tight upper bound on $\prob{\tuple{z}\in\calR_1} \leq \min\{\eta_1,\eta_2\}$, we obtain the main result.
\end{proof}

\subsection{Proof of Theorem~\ref{thm:bounds}}
\label{app:proof-bounds}
\subsubsection{Useful Lemmas}
We first state and prove a few useful lemmas that are vital for proving Theorem~\ref{thm:bounds}. While some lemmas are trivial and have been proven elsewhere, we nonetheless state their proofs for completeness.

\begin{lemma}
    \label{lemma:convex-ub}
    Let $h: \reals^n \rightarrow \reals$ be a convex function, and $\calH(\calX) \subset \reals^n$ be the convex hull of $\calX = \{\vc{x}_1, ..., \vc{x}_d\}$ where $\{\vc{x}_i\}_{i=1}^d\in\reals^n$, then there exists $\vc{x}_m \in \calX$ such that:
\begin{equation}\label{eqn:lemma-ub}
    \max_{\vc{u}\in \calH(\calX)} h(\vc{u}) = h(\vc{x}_m)
\end{equation}
\end{lemma}

\begin{proof}
Let $\vc{u}$ be any arbitrary vector in $\calH(\calX)$, that is $\exists \bm{\alpha} \in \Delta_d$:
\begin{equation}
    \vc{u} = \sum_{i=1}^d \alpha_i \vc{x}_i
\end{equation}

Let $m\in [d]$ such that $h(\vc{x}_m)\geq h(\vc{x}_i)$ $\forall i\in [d]$. From the convexity of $h$, we upper bound $h(\vc{u})$ as follows:
\begin{equation}
    h(\vc{u}) =  h\left(\sum_{i=1}^d\alpha_i \vc{x}_i\right) \leq  \sum_{i=1}^d \alpha_i h(\vc{x}_i) \leq  \sum_{i=1}^d \alpha_i h(\vc{x}_m) =   h(\vc{x}_m) \sum_{i=1}^d \alpha_i = h(\vc{x}_m)
\end{equation}
Thus, \eqref{eqn:lemma-ub} holds for any $\vc{u}\in \calH(\calX)$. 
%and is achieved when $\vc{u} = \vc{x}_m$.
\end{proof}

\begin{lemma}[Redistribution Lemma] \label{lemma:re-distribute}
$\forall p,q$ such that  $0\leq p\leq q \leq 1$, $\forall \bm{\alpha}\in \Delta_M$, and $\forall \calI, \calJ \subseteq [M]$ such that $\calI \notin \calJ \notin \calI$ we have:
\begin{equation}
    p \cdot \max_{i \in \calI}\{\alpha_i\} + q \cdot \max_{j\in \calJ} \{\alpha_j\} \geq  p \cdot \max_{i \in \calI \cup \calJ}\{\alpha_i\} + (q-p) \cdot \max_{j\in \calJ} \{\alpha_j\} +p \cdot \max_{k\in \calI \cap \calJ} \{\alpha_k\}
\end{equation}
\end{lemma}
\begin{proof}

\begin{align}
\begin{split}
    p \cdot \max_{i \in \calI}\{\alpha_i\} + q \cdot \max_{j \in \calJ}\{\alpha_j\} &= p \cdot \alpha_{i^*} + q \cdot \alpha_{j^*} \\
    &= p\cdot (\alpha_{i^*} + \alpha_{j^*}) + (q-p) \cdot \alpha_{j^*} \\
    &\myeq{(a)} p\cdot \left( \max_{i \in \calI\cup\calJ}\{\alpha_i\} + \min\{\alpha_{i^*},\alpha_{j^*}\}\right) + (q-p) \cdot \alpha_{j^*} \\
    &\mygeq{(b)} p\cdot \max_{i \in \calI\cup\calJ}\{\alpha_i\} + (q-p) \cdot \alpha_{j^*} + p \cdot \max_{k \in \calI \cap \calJ}\{\alpha_k\} \\
    &= p \cdot \max_{i \in \calI\cup\calJ}\{\alpha_i\} + \left(q-p\right)\cdot \max_{j \in \calJ}\{\alpha_j\} +p \cdot \max_{k \in \calI \cap \calJ}\{\alpha_k\}
\end{split}
\end{align}
where (a) holds because the maximum over $\calI \cup \calJ$ is either $\alpha_{i^*}$ or $\alpha_{j^*}$, and (b) holds since the smallest of the two maximizers cannot be smaller than the maximizer of the smaller set $\calI \cap \calJ$.
\end{proof}

\begin{lemma} \label{lemma:bound-eta}
Let $\{f_i\}_{i=1}^M$ be an arbitrary collection of $C$-ary classifiers with individual adversarial risks $\eta_i$ such that $0<\eta_1\leq ... \leq \eta_M\leq 1$. For any data distribution $\calD$ and perturbation set $\calS$ we have $\forall \bm{\alpha}\in \Delta_M$:
\begin{equation}
    \eta(\bm{\alpha}) \geq \sum_{i=1}^M \left(\left( \eta_i - \eta_{i-1}\right) \cdot \max_{j\in \{i,...,M\}}\{\alpha_j\} \right)
\end{equation}
where $\eta_{0}\doteq0$. 
\end{lemma}
%Note that $\tilde{\eta}$ can be interpreted as the robust error of an auxiliary ensemble $\tilde{\calF}$ with the same individual robust errors $\tilde{\eta}_i = \eta_i$ such that the classifiers are always co-robust and their corresponding sets of mis-classified samples are contained in each-other: $\tilde{\calT}_1 \subseteq \tilde{\calT}_2 \subseteq ... \subseteq \tilde{\calT}_M$. Also note that the existence of such an ensemble is not guaranteed, depending on the properties of $\calS$, e.g. if $\calS = \{\bm{\delta}\}$ then it is impossible for any two classifiers to be co-robust.
\begin{proof} From Proposition~\ref{prop:eta} we know that  $\exists K\in \naturals$, $\vc{p}\in \Delta_K$, and $\calU_k\subseteq\{0,1\}^M$ $\forall k\in[K]$ such that:
\begin{equation}
    \eta(\bm{\alpha}) = \sum_{k=1}^K\left(p_k \cdot \max_{\vc{u}\in \calU_k}\left\{\tp{\vc{u}}\bm{\alpha}\right\}\right)
\end{equation}
Let $\calL_k \subseteq [M]$ represent the set of classifier indices $i_1, ..., i_n$ that are active in the configuration $\calU_k$, that is:
\begin{equation}
    m \in \calL_k \iff \exists \vc{v} \in \calU_k \text{ such that }v_{m} = 1
\end{equation}
We then lower bound $\eta$ as follows:
\begin{equation}
    \eta(\bm{\alpha}) = \sum_{k=1}^K\left(p_k \cdot \max_{\vc{u}\in \calU_k}\left\{\tp{\vc{u}}\bm{\alpha}\right\}\right) \geq \sum_{k=1}^K\left(p_k \cdot \max_{i\in \calL_k}\left\{\alpha_i\right\}\right) = \eta'(\bm{\alpha})
\end{equation}
The bound trivially holds, since the sum of positive numbers is always larger than any summand. It is noteworthy to point out that the RHS quantity $\eta'(\bm{\alpha})$ can be interpreted as the adversarial risk of an \emph{auxiliary} set of classifiers $\calF'$ with same individual risks $\{\eta_i\}$ such that for any $\tuple{z}\in \reals^d\times [C]$, the classifiers have no common adversarial perturbations, i.e.:
\begin{equation}
    \bigcap_{i=1}^M \calS'_i(\tuple{z}) = \varnothing
\end{equation}
and:
\begin{equation}
    \eta'_i = \eta'(\vc{e}_i) = \sum_{k: i\in \calL_k} p_k = \eta(\vc{e}_i) = \eta_i    
\end{equation}

Assume that the conditions of Lemma~\ref{lemma:re-distribute} are met by two terms in $\eta'$, i.e., $\exists k_1,k_2 \in [K]$ such that $\calL_{k_1} \notin \calL_{k_2}\notin \calL_{k_1}$ and $p_{k_1}\leq p_{k_2}$, then we can apply the bound in Lemma~\ref{lemma:re-distribute} and obtain:
\begin{align}
    \begin{split}
        \eta'(\bm{\alpha}) &- \sum_{k\in [K]\setminus \{k_1,k_2\}}\left(p_k \cdot \max_{i\in \calL_k}\left\{\alpha_i\right\}\right) =    p_{k_1} \cdot \max_{i \in \calL_{k_1}}\{\alpha_i\} + p_{k_2} \cdot \max_{i\in \calL_{k_2}} \{\alpha_i\} \\
        &\geq   p_{k_1} \cdot \max_{i \in \calL_{k_1}\cup\calL_{k_2}}\{\alpha_i\} + \left(p_{k_2}-p_{k_1}\right)\cdot \max_{j \in \calL_{k_2}}\{\alpha_j\} +p_{k_1} \cdot \max_{k \in \calL_{k_1} \cap \calL_{k_2}}\{\alpha_k\} \\
        &= \eta''(\bm{\alpha}) - \sum_{k\in [K]\setminus \{k_1,k_2\}}\left(p_k \cdot \max_{i\in \calL_k}\left\{\alpha_i\right\}\right)
    \end{split}
\end{align}
where $\eta''(\bm{\alpha})$ is the modified ensemble adversarial risk. The application of Lemma~\ref{lemma:re-distribute} can be understood as a way to \say{re-distribute} the classifiers' adversarial vulnerabilities while  preserving the adversarial risk identities $\forall i\in[M]$:
\begin{equation}
      \eta_i= \eta'(\vc{e}_i) = \sum_{k: i\in \calL_k} p_k = \eta''(\vc{e}_i)
\end{equation}
The main idea of this proof is to keep applying Lemma~\ref{lemma:re-distribute} to the modified ensemble adversarial risks (if possible) to obtain a better lower bound. The process stops when the conditions are no longer met, and we obtain an adversarial risk $h(\bm{\alpha})$:
\begin{equation}
   \eta'(\bm{\alpha}) \geq \eta''(\bm{\alpha}) \geq .. \geq h(\bm{\alpha}) = \sum_{k=1}^{L}\left( q_k \cdot \max_{j\in \calJ_k}\left\{\alpha_j\right\}\right)
\end{equation}
Without loss of generality, we will assume that $\{\calJ_k\}$ are distinct and $q_k\neq 0$. Furthermore, since the conditions of Lemma~\ref{lemma:re-distribute} cannot be met by any two sets in $\{\calJ_k\}$, we must have (up to a re-ordering of the indices):
\begin{equation} \label{eq:set-ordering}
    \calJ_{L} \subset \calJ_{L-1} \subset ... \subset \calJ_{1} \subseteq [M]
\end{equation}
We now make the following observations:
\begin{enumerate}
    \item Due to \eqref{eq:set-ordering}, we have that $L\leq M$ and for all $i\in [M]$, $\exists m_i \in [L]$ such that:
    \begin{equation}
        \eta_i = \sum_{k: i\in \calJ_k}q_k = \sum_{k=1}^{m_i}q_k
    \end{equation}
    \item Since $\{\eta_i\}$ are sorted, we get that $m_{i+1} = m_i + 1$ if $\eta_i < \eta_{i+1}$ or $m_{i+1}=m_i$ otherwise
    \item $\calJ_1 = [M]$ since $\eta_1 \neq 0$ 
    \item For any two consecutive sets $\calJ_k$ and $\calJ_{k+1}$, we can always find $n\geq 1$ indices from $[M]$ such that $\calJ_k = \calJ_{k+1}\cup \{i_1, ..., i_n\}$. The indices $i_1, ..., i_n$ are consecutive, share the same $m_i$ (i.e., $\eta_{i_l}$ is the same for all $l\in[n]$), and also satisfy:
    \begin{equation}
        \min_{l\in [n]}\{i_l\} =   \max_{j \in \calJ_{k+1}}\{j\} + 1         
    \end{equation}
\end{enumerate}

We first prove the lemma for the special case of distinct risks, i.e. $\eta_i < \eta_{i+1}$ $\forall i$.

\textbf{Special Case}. The risks are distinct, then we must have $L=M$, with every two consecutive sets $\calJ_k$ and $\calJ_{k+1}$ differing by one index. Therefore we have $\calJ_k = \calJ_{k+1}\cup \{k\}$ and $\calJ_{M+1}=\varnothing$. Furthermore, we will get $\eta_{i}-\eta_{i-1} = q_{i}$ $\forall i \in [M]$ with $\eta_0 = 0$. Thus we can write:
\begin{align}
    \begin{split}
    h(\bm{\alpha}) = \sum_{k=1}^{M}\left( q_k \cdot \max_{j\in \calJ_k}\left\{\alpha_j\right\}\right)= \sum_{i=1}^M \left(\left( \eta_i - \eta_{i-1}\right) \cdot \max_{j\in \{i,...,M\}}\{\alpha_j\} \right)
    \end{split}
\end{align}

\textbf{General Case}. For the general case
we will have $L\leq M$ distinct risks $\eta_{i_1}< ...<\eta_{i_L}$ and $M-L$ repeated risks, where $i_1=1$. Thus we have $q_k = \eta_{i_{k}}-\eta_{i_{k-1}}$ $\forall k \in[L]$, and $\eta_{i_0}=\eta_0 = 0$ by definition. Using observations 3 and 4, we have that $\calJ_k = \{u_k,...,M\}$ for some index $u_k \in [M]$, with $u_1=1$. Thus we have $u_{k+1}-u_k-1\geq0$ to be the number of of consecutive repeated risks equal to $\eta_{i_k}$. Let $\{\calJ'_k\}$ be the $M-L$ index sets missing from $\{i\in [M]: \{i,...,M\}\}$, then we have:
\begin{align}
    \begin{split}
    h(\bm{\alpha}) &= \sum_{k=1}^{L}\left( q_k \cdot \max_{j\in \calJ_k}\left\{\alpha_j\right\}\right)\\ &= \sum_{k=1}^{L}\left( \left(\eta_{i_{k+1}}-\eta_{i_k}\right) \cdot \max_{j\in \{u_k, ..., M\}}\left\{\alpha_j\right\}\right)\\
    &= \sum_{k=1}^{L}\left( \left(\eta_{i_{k+1}}-\eta_{i_k}\right) \cdot \max_{j\in \{u_k, ..., M\}}\left\{\alpha_j\right\}\right) + \sum_{k=1}^{M-L}\left( 0 \cdot \max_{j\in \calJ'_k}\left\{\alpha_j\right\}\right)\\
    &\myeq{(a)} \sum_{i=1}^M \left(\left( \eta_i - \eta_{i-1}\right) \cdot \max_{j\in \{i,...,M\}}\{\alpha_j\} \right)
    \end{split}
\end{align}
where (a) holds due to the fact $\eta_i - \eta_{i-1} = 0$ for all the merged $M-L$ terms.
\end{proof}
\begin{lemma} \label{lemma:min-bound}
Given a sequence $\{\gamma_i\}_{i=0}^M$
such that $0=\gamma_0<\gamma_1 \leq ...\leq \gamma_M\leq1$, the vector $\bm{\alpha}^*=\tp{\left[\frac{1}{m}\ ...\ \frac{1}{m}\ 0\ ...\ 0\right]} \in \Delta_M$ is a solution to the following minimization problem:
%$\forall \gamma_i$ such that $0=\gamma_0<\gamma_1 \leq ...\leq \gamma_M\leq1$, define $k\in [M]$ such that $\frac{\gamma_k}{k} \leq \frac{\gamma_i}{i}$, $\forall i \in [M]$. Then, the vector $\bm{\alpha}^*=\tp{\left(\frac{1}{k},...,\frac{1}{k},0,...,0\right)} \in \Delta_M$ is an optimal solution for:
\begin{equation}\label{eq:cor-opt}
    \min_{\bm{\alpha}\in \Delta_M} h(\bm{\alpha})=\min_{\bm{\alpha}\in \Delta_M} \sum_{i=1}^M \left(\left(\gamma_i-\gamma_{i-1}\right)\cdot \max_{j\in \{i,...,M\}}\{\alpha_j\}\right) = \frac{\gamma_m}{m}
 \end{equation}
 where $\frac{\gamma_m}{m} \leq \frac{\gamma_i}{i}$, $\forall i \in [M]$.
\end{lemma}
\begin{proof}
 We know that $h$ is a piece-wise linear convex function over a closed and convex set, which implies the existence of a global minimizer.

Define the mapping $g: \Delta_M \rightarrow [0,1]^M$ such that $\forall i\in [M]$:
\begin{equation}
    g_i(\bm{\alpha}) = \max_{j\in\{i,...,M\}}\{\alpha_j\} - \max_{j\in \{i+1,...,M\}}\{\alpha_j\} 
\end{equation}

We can re-write the function $h$ via a simple re-arrangement to obtain:
\begin{equation}
    h(\bm{\alpha}) = \sum_{i=1}^{M}\gamma_i \cdot \left(\max_{j\in\{i,...,M\}}\{\alpha_j\}-\max_{j\in\{i+1,...,M\}}\{\alpha_j\}\right) = \sum_{i=1}^{M}\gamma_i \cdot g_i(\bm{\alpha}) = \tp{\bm{\gamma}}g(\bm{\alpha})
\end{equation}

Define the decomposition over the probability simplex: $\Delta_M = \Delta^1_M \cup \Delta^2_M \cup ... \cup \Delta^{M!}_M$, where $\forall n \in [M!]$, $\exists i_1, i_2, ..., i_M$ such that $\forall \bm{\alpha}\in\Delta^n_M$ we have:
\begin{equation}
    \alpha_{i_1} \geq \alpha_{i_2} \geq \alpha_{i_3} \geq ... \geq \alpha_{i_M}
\end{equation}
In other words, $\Delta^n_M$ is the set of all probability vectors that share the same sorting indices. Since we have $M!$ ways to arrange $M$ numbers, the size of the decomposition will be $M!$.
We now make the following observations:

\textbf{1}. $\forall n$, $\Delta^n_M$ is a convex set. \textit{quick proof}: Let $\bm{\alpha},\bm{\beta} \in \Delta^n_M$, then $\exists i_1, i_2, ..., i_M$ such that $\alpha_{i_1}\geq ... \geq \alpha_{i_M}$ and $\beta_{i_1} \geq ... \geq \beta_{i_M}$. $\forall \lambda \in [0,1]$ we have $\vc{q} = \lambda \bm{\alpha} + (1-\lambda)\bm{\beta} \in \Delta_M$, since $\sum q_i = \sum \lambda \alpha_i + (1-\lambda)\beta_i = 1$ and $q_i\geq 0$. We also have $\forall l\in [M-1]$:
\begin{equation}
    q_{i_l} = \lambda \alpha_{i_l} + (1-\lambda)\beta_{i_l} \geq \lambda \alpha_{i_{l+1}} + (1-\lambda)\beta_{i_{l+1}} = q_{i_{l+1}}
\end{equation}

\textbf{2}. $\forall n$, $\exists \calP^n=\{ \vc{p}^n_1,...,\vc{p}^n_M\} \subset \Delta^n_M$ such that $\Delta^n_M = \calH(\calP^n)$, where $\calH(\calX)$ is the convex hull of the set of points $\calX$. \textit{quick proof}: Let $i_1,...,i_M$ be the sorted indices associated with an arbitrary subset $\Delta^n_M$. Construct the $M$ probability vectors as follows: $\forall k\in [M]$ $p^n_{k,j} = \frac{1}{k}$ if $j\in\{i_1,...,i_{k}\}$ else $p^n_{k,j} = 0$. It is easy to verify that $\vc{p}^n_k\in\Delta^n_M$, since $\sum_jp^n_{k,j}=k/k=1$, and $p^n_{k,i_1}\geq ... \geq p^n_{k,i_M}$. Since $\Delta^n_M$ is convex (Claim 1), we thus have that $\calH(\calP^n) \subseteq \Delta^n_M$. What is left is to show that $\Delta^n_M\subseteq\calH(\calP^n)$, which can be established if we show that $\forall \bm{\alpha}\in \Delta^n_M$, $\exists \bm{\lambda}\in \Delta_M$ such that $\bm{\alpha} = \sum_k \lambda_k \vc{p}^n_{k}$. We shall prove it by construction, specifically define:
\begin{equation}
    \lambda_{k} = k\cdot(\alpha_{i_k}-\alpha_{i_{k+1}}) \geq 0
\end{equation}
This induces a valid convex coefficient vector $\bm{\lambda}$, since $\sum_k \lambda_k = \sum_k (\alpha_{i_k} - \alpha_{i_{k+1}})\cdot k = \sum_k \alpha_{i_k}  = 1$. It is also easy to verify that $\alpha_{i_l} = \sum_k \lambda_k p^n_{k,i_l}$ for all indices $i_l\in[M]$, since:
\begin{align}
\begin{split}
    \sum_{k=1}^M \lambda_k p^n_{k,i_l} &= \frac{M}{M}\cdot(\alpha_{i_M}-0) +\frac{M-1}{M-1}\cdot (\alpha_{i_{M-1}}-\alpha_{i_M}) + ... + \frac{l}{l}\cdot(\alpha_{i_{l}}-\alpha_{i_{l+1}}) = \alpha_{i_l}
\end{split}
\end{align}
by construction of $\bm{\lambda}$ and $\calP^n$.

\textbf{3}. $\forall n$, the function $g$ is linear over $\bm{\alpha}\in \Delta^n_M$.
\textit{quick proof}: Define the maximum index $s(\bm{\alpha},i) = \argmax_{j\in\{i,...,M\}}\{\alpha_j\}$.
By definition, $\bm{\alpha} \in \Delta^n_M$ implies that $s(i)=s(\bm{\alpha},i)$ is independent of $\bm{\alpha}$. Therefore $\forall i\in [M]$ we have $g_i(\bm{\alpha}) = \alpha_{s(i)}-\alpha_{s(i+1)}$ with the slight abuse of notation $\alpha_{M+1}=0$. Therefore $\exists \mtx{G}^n\in\{-1,0,1\}^{M\times M}$ such that $g(\bm{\alpha}) = \mtx{G}^n \bm{\alpha}$ for all $\bm{\alpha}\in\Delta^n_M$.

Combining observations 1,2\&3, we can re-write the original optimization problem as follows:
\begin{align}
    \label{eq:h-123}
    \begin{split}
    \min_{\bm{\alpha}\in \Delta_M} h(\bm{\alpha})&=  \min_{\bm{\alpha}\in \Delta^1_M \cup ... \cup \Delta^{M!}_M} \tp{\bm{\gamma}}g(\bm{\alpha}) \\
    &= \min_{n\in[M!]}\left\{\min_{\bm{\alpha}\in\Delta^n_M}\tp{\bm{\gamma}}g(\bm{\alpha})\right\}\\
    &= \min_{n\in[M!]}\left\{\min_{\bm{\alpha}\in\calH(\calP^n)}\tp{\bm{\gamma}}\mtx{G}^n\bm{\alpha}\right\} \\
    &\myeq{(a)} \min_{n\in[M!]}\left\{\min_{\vc{p}\in\calP^n}\tp{\bm{\gamma}}\mtx{G}^n\vc{p}\right\} \\
    &= \min_{n\in[M!], k\in [M]} \tp{\bm{\gamma}}g(\vc{p}^n_k)
    \end{split}
\end{align}
where (a) holds because the minimum of a linear function over the convex hull of a set of points $\calX$ is obtained at one of the points in $\calX$. 

Thus, to solve the original optimization problem, we only need to evaluate $M!$ linear functions with $M$ vectors each, and pick the one that achieves the smallest value. Finally, we will now show that the search space can be significantly reduced from $M!\times M$ to $M$ possible solutions. 

Let $\Delta^n_M$ be an arbitrary subset of $\Delta_M$ whose associated sorted indices are $i^n_1, i^n_2, ..., i^n_M$, and $\calP^n=\{\vc{p}^n_k\}_k$ are the associated extreme points. We first note that, $\forall k\in [M]$, $g(\vc{p}^n_k) = \tp{\left[0\ ...\ 0\ \frac{1}{k}\ 0\ ...\ 0\right]}$ with $j^n_k = \max\{i^n_1,...,i^n_k\}$ is the non-zero index. Therefore, we have that $\forall n,k$:
\begin{equation}
    \label{eq:h-eval}
    h(\vc{p}^n_k) = \tp{\bm{\gamma}}g(\vc{p}^n_k)=\frac{\gamma_{j^n_k}}{k}
\end{equation}
Equation \eqref{eq:h-eval} reveals that, amongst all vectors $\vc{p}^n_k$ with fixed $k$, the smallest error is always achieved by the subset $n$ whose associated $j^n_k$ index is the smallest, since the robust errors are always assumed to be sorted. Furthermore, the smallest value that $j^n_k$ can achieve is $k$, since it is the largest index amongst $k$ arbitrary indices from $[M]$. Therefore, let $\Delta^m_M$ be the subset whose sorting indices are $i_k=k$, i.e. $\bm{\alpha}\in\Delta^m_M$ implies $\alpha_1\geq...\geq\alpha_M$. For this subset, we will always have $j^m_k=\max\{1,...,k\}=k$ which implies that $\forall n\in[M!]$ and $\forall k\in[M]$:
\begin{equation}
    \label{eq:h-ineq}
    h(\vc{p}^n_k) = \frac{\gamma_{j^n_k}}{k} \geq \frac{\gamma_{k}}{k} = h(\vc{p}^m_k)
\end{equation}
where $\vc{p}^m_k = \tp{\left[\frac{1}{k}\ ...\ \frac{1}{k}\ 0\ ...\ 0\right]}$. Combining \eqref{eq:h-123}\&\eqref{eq:h-ineq} we obtain:
\begin{align}
    \begin{split}
    \min_{\bm{\alpha}\in \Delta_M} h(\bm{\alpha}) = \min_{k\in [M]} \tp{\bm{\gamma}}g(\vc{p}^m_k) = \min_{k\in[M]} \frac{\gamma_k}{k} = \frac{\gamma_{k^*}}{k^*}
    \end{split}
\end{align}
which can be achieved using $\bm{\alpha}^*=\tp{\left[\frac{1}{k^*}\ ...\ \frac{1}{k^*}\ 0\ ...\ 0\right]}$.
\end{proof}
\subsubsection{Main Proof}
We now restate and prove Theorem~\ref{thm:bounds}:
\begin{theorem*}[Restated]
For a perturbation set $\calS$, data distribution $\calD$, and collection of $M$ classifiers $\calF$ with individual adversarial risks $\eta_i$ ($i\in [M]$) such that $0<\eta_1\leq ... \leq \eta_M\leq 1$, we have $\forall \bm{\alpha}\in \Delta_M$:
\begin{equation}\label{eqn:bounds}
       \min_{k\in [M]} \left\{\frac{\eta_k}{k}\right\}  \leq  \eta(\bm{\alpha}) \leq \eta_M  
\end{equation}
Both bounds are tight in the sense that if all that is known about the setup $\calF$, $\calD$, and $\calS$ is $\{\eta_i\}_{i=1}^M$, then there exist no tighter bounds. Furthermore, the upper bound is always met if $\bm{\alpha}=\vc{e}_M$, and the lower bound (if achievable) can be met if $\bm{\alpha}=\tp{\left[\frac{1}{m}\ ...\ \frac{1}{m}\ 0\ ...\ 0\right]}$, where $m=\argmin_{k\in [M]} \{\frac{\eta_k}{k}\}$.
\end{theorem*}

\begin{proof} We first prove the upper bound and then the lower bound.

\textbf{Upper bound}: From Proposition~\ref{prop:eta}, we have that $\eta$ is convex in $\bm{\alpha} \in \Delta_M$. Using  $\Delta_M = \calH\left(\{\vc{e}_1, ..., \vc{e}_M\}\right)$ and applying Lemma~\ref{lemma:convex-ub}, we get $\forall \bm{\alpha}\in\Delta_M$:
\begin{equation}
    \eta(\bm{\alpha})\leq \max_{\bm{\alpha}\in \Delta_M} \eta(\bm{\alpha}) = \max_{i\in [M]} \eta(\vc{e}_i) = \eta_M
\end{equation}
This establishes the upper bound in \eqref{eqn:bounds}. The bound is tight, since $\eta(\vc{e}_M) = \eta_M$ is achievable.

\textbf{Lower bound}: From Lemmas~\ref{lemma:bound-eta}\&\ref{lemma:min-bound}, we establish $\forall \bm{\alpha}\in \Delta_M$, the following result:
\begin{equation}
    \eta(\bm{\alpha}) \geq \sum_{i=1}^M \left(\left( \eta_i - \eta_{i-1}\right) \cdot \max_{j\in \{i,...,M\}}\{\alpha_j\} \right) = h(\bm{\alpha}) \geq \min_{\bm{\alpha}\in \Delta_M} h(\bm{\alpha}) = h(\bm{\alpha}^*) = \frac{\eta_m}{m}
\end{equation}
where $m = \argmin_{k\in[M]} \{\frac{\eta_k}{k}\}$ and $\bm{\alpha}^*=\tp{\left[\frac{1}{m}\ ...\ \frac{1}{m}\ 0\ ...\ 0\right]}$. This establishes the lower bound in \eqref{eqn:bounds}. 

The bound is tight, since for fixed $0<\eta_1\leq ...\leq \eta_M\leq 1$, we can construct $\calF$, $\calS$, and $\calD$ such that $\eta(\bm{\alpha}) = h(\bm{\alpha})$ and $\forall i\in[M]: \eta(\vc{e}_i)=h(\vc{e}_i)=\eta_i$, as shown next.

Let $\calS \subset \reals^d$ be any closed and bounded set containing at least $M$ distinct vectors $\{\bm{\delta}_j\}_{j=1}^M \subseteq \calS$. Let $\calD$ be any valid distribution over $\calR = \reals^d \times [C]$ such that $\forall i \in [M]$: $\prob{\tuple{z} \in \calT_i} = \eta_i$, $\prob{\tuple{z} \in \calT_{M+1}} = 1$, and $\varnothing=\calT_0\subset \calT_1 \subseteq \calT_2 \subseteq ... \subseteq \calT_M \subseteq \calT_{M+1} \subset \calR$. Finally, we construct classifiers $f_i$ ($\forall i\in [M]$) to satisfy the following assignment $\forall \tuple{z}\in \calT_{M+1}$:
\begin{equation}
     f_i(\vc{x}+\bm{\delta}) = y \ \ \ \forall \bm{\delta}\in \calS\setminus\{\bm{\delta}_i\} \ \ \ \ \ \text{\&} \ \ \ \ \  f_i(\vc{x}+\bm{\delta}_i) = 
     \begin{cases}
      y & \text{if $(\vc{x},y) \notin \calT_i$}\\
      y' \neq y & \text{otherwise}
    \end{cases}  
\end{equation}

 i.e., the $i$-th classifier decision $f_i(\vc{x}+\bm{\delta})$ is incorrect only if $\bm{\delta}=\bm{\delta}_i$ and $\tuple{z}\in\calT_i$.

%Finally, let $\calD$ be any valid distribution over $\reals^d \times [C]$ such that $\forall i\in [M]$:  $\prob{\tuple{z} \in \calT_i} = \eta_i$.
Given the above construction, we establish %\ns{$\calT_0$ needs to be defined for (a).}:
\begin{align}
    \begin{split}
        \eta(\bm{\alpha}) &= \means{\tuple{z}\sim \calD}{\max_{\bm{\delta}\in\calS}{ \sum_{i=1}^M \alpha_i \identityf{f_i(\vc{x}+\bm{\delta})\neq y}}} \\
        &\myeq{(a)} \int_{\tuple{z}\in\calT_{M+1}} p_z(\tuple{z}) \cdot \left( \max_{\bm{\delta}\in\calS}{ \sum_{i=1}^M \alpha_i \identityf{f_i(\vc{x}+\bm{\delta})\neq y}}\right) \,d\tuple{z}\\
        &\myeq{(b)} \sum_{k=1}^M \int_{\tuple{z}\in\calT_k\setminus\calT_{k-1}} p_z(\tuple{z}) \cdot \left( \max_{\bm{\delta}\in\calS}{ \sum_{i=1}^M \alpha_i \identityf{f_i(\vc{x}+\bm{\delta})\neq y}}\right) \,d\tuple{z}\\
        &\myeq{(c)} \sum_{k=1}^M \int_{\tuple{z}\in\calT_k\setminus\calT_{k-1}} p_z(\tuple{z}) \cdot \left( \max_{j \in \{k,..,M\}} \{\alpha_j\}\right) \,d\tuple{z} \\
        &= \sum_{k=1}^M \left[\left( \max_{j \in \{k,..,M\}} \{\alpha_j\}\right) \int_{\calT_k\setminus\calT_{k-1}} p_z(\tuple{z})  \,d\tuple{z}\right] \\
        &\myeq{(d)} \sum_{k=1}^M \left[ (\eta_i - \eta_{i-1}) \cdot \max_{j \in \{k,..,M\}} \{\alpha_j\}\right] = h(\bm{\alpha})
    \end{split}
\end{align}
where: (a) holds because $\prob{\tuple{z}\in \calT_{M+1}} = 1$; (b) holds because we can partition  $\calT_{M+1}$ into $M+1$ sets:
$\calT_1\cup (\calT_2\setminus\calT_{1}) \cup \ldots \cup (\calT_{M+1}\setminus\calT_{M})$, and because the max term is 0 $\forall\tuple{z}\in\calT_{M+1}\setminus\calT_{M}$; (c) holds by construction of $\calF$ and $\calS$, and (d) holds since $\eta_i = \prob{\tuple{z}\in\calT_i}$ and $\calT_i \subseteq \calT_{i+1}$.

%$\calT_1$, $\calT_2\setminus\calT_1$, ..., $\calT_M\setminus\calT_{M-1}$ and $\calR,\setminus \calT_M$, and $\tuple{z}\in\calR \setminus \calT_M$ implies the max term is always 0, (b) holds by construction of $\calF$ and $\calS$, and (c) holds since $\eta_i = \prob{\tuple{z}\in\calT_i}$ and $\calT_i \subseteq \calT_{i+1}$.
\end{proof}

\subsection{Proof of Theorem~\ref{thm:three-class}}
In this section, we derive a simplified search strategy for finding the optimal sampling probability for the special case of $M=3$, akin to Section~\ref{ssec:two}. 
\begin{theorem*}[Restated] Define $\calA \subset \Delta_3$ to be the set of the following vectors:
\begin{equation}
    \calA = \left\{ \begin{bmatrix}
1\\
0 \\
0
\end{bmatrix}, \begin{bmatrix}
0\\
1 \\
0
\end{bmatrix},
\begin{bmatrix}
0\\
0 \\
1
\end{bmatrix},
\begin{bmatrix}
\nicefrac{1}{2}\\
\nicefrac{1}{2} \\
0
\end{bmatrix},
\begin{bmatrix}
0\\
\nicefrac{1}{2} \\
\nicefrac{1}{2}
\end{bmatrix},
\begin{bmatrix}
\nicefrac{1}{2} \\
0\\
\nicefrac{1}{2}
\end{bmatrix},
\begin{bmatrix}
\nicefrac{1}{2} \\
\nicefrac{1}{4}\\
\nicefrac{1}{4}
\end{bmatrix}
,
\begin{bmatrix}
\nicefrac{1}{4} \\
\nicefrac{1}{2}\\
\nicefrac{1}{4}
\end{bmatrix},
\begin{bmatrix}
\nicefrac{1}{4} \\
\nicefrac{1}{4}\\
\nicefrac{1}{2}
\end{bmatrix},
\begin{bmatrix}
\nicefrac{1}{3} \\
\nicefrac{1}{3}\\
\nicefrac{1}{3}
\end{bmatrix}
\right\}
\end{equation}
Then for any three classifiers $f_1$, $f_2$, and $f_3$, perturbation set $\calS \subset \reals^d$, and data distribution $\calD$, we have:
\begin{equation}\label{eq:thm3-min-v2}
    \min_{\bm{\alpha} \in \Delta_3} \eta(\bm{\alpha})  =  \min_{\bm{\alpha} \in \calA} \eta(\bm{\alpha})
\end{equation}
The set $\calA$ is optimal, in the sense that there exist no smaller set $\calA'$ such that \eqref{eq:thm3-min-v2} holds.
\end{theorem*}

\begin{proof}
Similar to \eqref{eq:eta-m2}, we can enumerate all possible classifiers/data-point configurations around $\calS$, which allows us to write $\forall \bm{\alpha} \in \Delta_3$:
\begin{align}\label{eq:eta-m3}
    \begin{split}
        \eta(\bm{\alpha}) &= p_1 \cdot \max \{\alpha_1, \alpha_2, \alpha_3\} \\
        &+ p_2 \cdot \max \{\alpha_1 +  \alpha_2, \alpha_3\} + p_3 \cdot \max \{\alpha_2 +    \alpha_3, \alpha_1\} + p_4 \cdot \max \{\alpha_1 +  \alpha_3, \alpha_2\} \\
        &+ p_5 \cdot \max\{\alpha_1, \alpha_2\} + p_6 \cdot \max\{\alpha_2, \alpha_3\} +  p_7 \cdot \max\{\alpha_1, \alpha_3\} \\
        &+ p_8 \cdot \alpha_1 + p_9 \cdot \alpha_2 + p_{10}\cdot \alpha_3 + p_{11} \cdot 1 + p_{12} \cdot 0
    \end{split}
\end{align}
where $\vc{p}\in \Delta_{12}$. 
%Theorem~\ref{thm:three-class} simplifies the search for the optimal sampling probability significantly, as it is sufficient to evaluate $\eta$ at exactly 10 different candidate solutions, captured by $\calA$, and pick the best performing one. Theorem~\ref{thm:three-class} also guarantees that the search procedure is efficient, since every candidate solution in $\calA$ needs to be evaluated.
We shall use the same technique used in the proof of Lemma~\ref{lemma:min-bound}. We can decompose $\Delta_3$ into 6 such subsets $\Delta_3^1, ..., \Delta_3^6$, such that each subset contains vectors that share the same sorting indices. These subsets are convex, and they can be represented as the convex hull of three vectors. Due to the symmetry of the problem, we shall focus on one subset $\Delta_3^1 = \calH\left(\left\{\tp{[1,0,0]}, \tp{[\nicefrac{1}{2},\nicefrac{1}{2},0]}, \tp{[\nicefrac{1}{3},\nicefrac{1}{3},\nicefrac{1}{3}]}\right\}\right)$ where $\forall \bm{\alpha} \in \Delta_3^1$, we have: $\alpha_1 \geq \alpha_2 \geq \alpha_3$. Notice that for any $\bm{\alpha} \in \Delta_3^1$, all the terms in \eqref{eq:eta-m3} become linear in $\bm{\alpha}$, except for the term $\max\{\alpha_2+\alpha_3,\alpha_1\}$. Therefore, we can further decompose $\Delta_3^1$ into two convex subsets $\Delta_3^{1,1}$ and $\Delta_3^{1,2}$, such that:
\begin{equation}
    \Delta_3^{1,1} = \{\bm{\alpha} \in \Delta_3^1: \alpha_1 \geq \alpha_2 + \alpha_3\} \ \ \ \ \Delta_3^{1,2} = \{\bm{\alpha} \in \Delta_3^1: \alpha_1 \leq \alpha_2 + \alpha_3\} 
\end{equation}
and $\eta$ is linear over both subsets (but not their union).

\textbf{Claim}: we have:
\begin{align}\label{eq:claim}
\begin{split}
     \Delta_3^{1,1} &= \calH\left(\left\{\tp{[1,0,0]}, \tp{[\nicefrac{1}{2},\nicefrac{1}{2},0]}, \tp{[\nicefrac{1}{2},\nicefrac{1}{4},\nicefrac{1}{4}]}\right\}\right) \\
    \Delta_3^{1,2} &= \calH\left(\left\{\tp{[\nicefrac{1}{3},\nicefrac{1}{3},\nicefrac{1}{3}]}, \tp{[\nicefrac{1}{2},\nicefrac{1}{2},0]}, \tp{[\nicefrac{1}{2},\nicefrac{1}{4},\nicefrac{1}{4}]}\right\}\right)   
\end{split}
\end{align}
Since both $\Delta_{3}^{1,1}$ and $\Delta_{3}^{1,2}$ are convex, it is enough to show that:
\begin{align}
\begin{split}
     \Delta_3^{1,1} &\subseteq \calH\left(\left\{\tp{[1,0,0]}, \tp{[\nicefrac{1}{2},\nicefrac{1}{2},0]}, \tp{[\nicefrac{1}{2},\nicefrac{1}{4},\nicefrac{1}{4}]}\right\}\right) \\
    \Delta_3^{1,2} &\subseteq \calH\left(\left\{\tp{[\nicefrac{1}{3},\nicefrac{1}{3},\nicefrac{1}{3}]}, \tp{[\nicefrac{1}{2},\nicefrac{1}{2},0]}, \tp{[\nicefrac{1}{2},\nicefrac{1}{4},\nicefrac{1}{4}]}\right\}\right)   
\end{split}
\end{align}
for \eqref{eq:claim} to hold. For all $\bm{\alpha}\in \Delta_3^{1,1}$, define:
\begin{equation}
    \lambda_1 = \alpha_1 - \alpha_2 - \alpha_3 \geq 0, \ \ \ \ \lambda_2 = 2\cdot(\alpha_2 - \alpha_3) \geq 0, \ \ \& \ \ \lambda_3 = 4\alpha_3 \geq 0
\end{equation}
Then we always have:
\begin{equation}
    \bm{\alpha} = \lambda_1 \cdot \begin{bmatrix}
    1\\ 0 \\0
    \end{bmatrix} + \lambda_2 \cdot \begin{bmatrix}
    \nicefrac{1}{2}\\ \nicefrac{1}{2} \\0
    \end{bmatrix} + \lambda_1 \cdot \begin{bmatrix}
    \nicefrac{1}{2}\\ \nicefrac{1}{4} \\\nicefrac{1}{4}
    \end{bmatrix} 
\end{equation}
where it is easy to verify that $\lambda_1 + \lambda_2 + \lambda_3 = 1$. The same can be shown for any $\bm{\alpha}\in \Delta_3^{1,2}$, using the following:
\begin{equation}
    \lambda_1 = 2\cdot (\alpha_2 - \alpha_3)\geq 0, \ \ \ \ \lambda_2 = 4\cdot(\alpha_1 - \alpha_2) \geq 0, \ \ \& \ \ \lambda_3 = 3 \cdot(\alpha_2 + \alpha_3 - \alpha_1) \geq 0
\end{equation}
which establishes the claim in \eqref{eq:claim}.

Using \eqref{eq:claim} and the linearity of $\eta$ on each subset, we can write:
\begin{align}
\begin{split}
\min_{\bm{\alpha} \in \Delta_3^1} \eta(\bm{\alpha}) &= \min\left\{\min_{\bm{\alpha}\in \Delta_3^{1,1}} \eta(\bm{\alpha}), \min_{\bm{\alpha}\in \Delta_3^{1,2}} \eta(\bm{\alpha})   \right\} \\&=   \min\left\{\eta\left(\begin{bmatrix}
    1\\ 0 \\0
    \end{bmatrix}\right),\eta\left(\begin{bmatrix}
    \nicefrac{1}{2}\\ \nicefrac{1}{2} \\0
    \end{bmatrix}\right),\eta\left(\begin{bmatrix}
    \nicefrac{1}{2}\\ \nicefrac{1}{4} \\\nicefrac{1}{4}
    \end{bmatrix}\right),\eta\left(\begin{bmatrix}
    \nicefrac{1}{3}\\ \nicefrac{1}{3} \\\nicefrac{1}{3}
    \end{bmatrix}\right)  \right\}  
\end{split}
\end{align}
Finally, repeating this procedure for the remainder 5 sets $\Delta_{3}^2, ..., \Delta_3^{6}$ establishes \eqref{eq:thm3-min}. To show that the set $\calA$ is minimal, we provide 10 constructions of $\eta$ using the $\vc{p}$ vector in \eqref{eq:eta-m3} such that the $i^{\text{th}}$ vector $\bm{\alpha} \in \calA$ is a unique (amongst $\calA$) global optimum of $\eta$ characterized by the $i^{\text{th}}$ $\vc{p}$ vector (listed below): %[\color{red}TODO fix a bug in this list\color{black}]
\begin{align}
\begin{split}
    \vc{p}_1 &= \tp{\left[0\ 0\ 0\ 0\ 0\ 0\ 0\ 0\ \frac{1}{2}\ \frac{1}{2}\ 0\ 0\right]} \\
    \vc{p}_2 &= \tp{\left[0\ 0\ 0\ 0\ 0\ 0\ 0\ \frac{1}{2}\ 0\ \frac{1}{2}\ 0\ 0\right]} \\
    \vc{p}_3 &= \tp{\left[0\ 0\ 0\ 0\ 0\ 0\ 0\ \frac{1}{2}\ \frac{1}{2}\ 0\ 0\ 0\right]} \\
    \vc{p}_4 &= \tp{\left[0\ 0\ 0\ 0\ \frac{1}{2}\ 0\ 0\ 0\ 0\ \frac{1}{2}\ 0\ 0\right]} \\
    \vc{p}_5 &= \tp{\left[0\ 0\ 0\ 0\ 0\ \frac{1}{2}\ 0\ \frac{1}{2}\ 0\ 0\ 0\ 0\right]} \\
    \vc{p}_6 &= \tp{\left[0\ 0\ 0\ 0\ 0\ 0\ \frac{1}{2}\ 0\ \frac{1}{2}\ 0\ 0\ 0\right]} \\
    \vc{p}_7 &= \tp{\left[0\ 0\ \frac{1}{2}\ 0\ 0\ \frac{1}{2}\ 0\ 0\ 0\ 0\ 0\ 0\right]} \\
    \vc{p}_8 &= \tp{\left[0\ 0\ 0\ \frac{1}{2}\ 0\ 0\ \frac{1}{2}\ 0\ 0\ 0\ 0\ 0\right]} \\
    \vc{p}_9 &= \tp{\left[0\ \frac{1}{2}\ 0\ 0\ \frac{1}{2}\ 0\ 0\ 0\ 0\ 0\ 0\ 0\right]} \\
    \vc{p}_{10} &= \tp{\left[1\ 0\ 0\ 0\ 0\ 0\ 0\ 0\ 0\ 0\ 0\ 0\right]} \\
\end{split}
\end{align}
\end{proof}

\subsection{Proof of Theorem~\ref{thm:osp}}
\label{app:proof-osp}
First, we state the classic result on the convergence of the projected sub-gradient method for convex minimization (\cite{shor2012minimization}):
\begin{lemma}[Projected Sub-gradient Method]\label{lemma:subgradient} Let $h: \reals^d \rightarrow \reals$ be a a convex and sub-differentiable function. Let $\calC \subset \reals^d$ be a convex set. For iterations $t=1,..,T$, define the projected sub-gradient method:
\begin{equation}
    \vc{x}^{(t+1)} = \proj{\calC}{\vc{x}^{(t)} - a_t \vc{g}^{(t)}} 
\end{equation}
\begin{equation}
    h_{\normalfont\text{best}}^{(t+1)} = \min \left\{h_{\normalfont\text{best}}^{(t)},h(\vc{x}^{(t+1)})\right\}
\end{equation}
where $a_t = \nicefrac{a}{t}$ for some positive $a>0$, $\vc{x}^{(1)} \in \calC$ is an arbitrary initial guess, $h_{\normalfont\text{best}}^{(1)} = h(\vc{x}^{(1)})$, and $\vc{g}^{(t)} \in \partial h(\vc{x}^{(t)})$ is a sub-gradient of $h$ at $\vc{x}^{(t)}$. Let $t_{\normalfont\text{best}}$ designate the best iteration index thus far. Then, if $h$ has norm-bounded sub-gradients: $\pnorm{\vc{g}}{2} \leq G$ for all $\vc{g}\in \partial h(\vc{x})$ and $\vc{x}\in \calC$, we have:
\begin{equation}
    h_{\normalfont\text{best}}^{(t)} - h^* \leq \frac{\pnorm{\vc{x}^{(1)}-\vc{x}^*}{2}^2 + G^2 \sum_{k=1}^t a_t^2 }{2 \sum_{k=1}^t a_k} \xrightarrow[t \to \infty]{} 0
\end{equation}
where:
\begin{equation}
    h^* = h(\vc{x}^*) = \min_{\vc{x} \in \calC} h(\vc{x})
\end{equation}
\end{lemma}
We then prove Theorem~\ref{thm:osp} (restated below) via a direct application of Lemma~\ref{lemma:subgradient}:
\begin{theorem*}[Restated]  The OSP algorithm output $\bm{\alpha}_{T}$ satisfies:
\begin{equation}
    0 \leq \hat{\eta}(\bm{\alpha}_{T}) - \hat{\eta}(\bm{\alpha}^*) \leq \frac{\pnorm{\bm{\alpha}^{(1)}-\bm{\alpha}^*}{2}^2 + M a^2\sum_{t=1}^T t^{-2} }{2 a\sum_{t=1}^T t^{-1}} \xrightarrow[T \to \infty]{} 0
\end{equation}
for any initial condition $\bm{\alpha}^{(1)} \in \Delta_M$, $a>0$, where $\bm{\alpha}^*$ is a global minimum.
\end{theorem*}
\begin{proof}
The ensemble empirical adversarial risk $\hat{\eta}$ is convex and sub-differentiable (Proposition~\ref{prop:eta}), being minimized over a convex set $\Delta_M$. At each iteration $t$ in OSP, the vector $\vc{g}$ obtained at line (12) is norm-bounded with $G=\sqrt{M}$, the vector $\vc{g}$ is also a sub-gradient of $\hat{\eta}$ at $\bm{\alpha}^{(t)}$, therefore Lemma~\ref{lemma:subgradient} applies.
\end{proof}
\subsection{Worst Case Performance of Deterministic Ensembles}\label{app:upper-bound}
In Section~\ref{ssec:bounds}, we showed via Theorem~\ref{thm:bounds} that the adversarial risk of any randomized ensemble classifier is upper bounded by the worst performing classifier in the ensemble $\calF$. In this section, we will show that the same cannot be said regarding deterministic ensemble classifiers. That is, there exist an ensemble $\calF$, data distribution $\calD$, and perturbation set $\calS$ such that:
\begin{equation}
    \eta(\bar{f}) > \max_{i\in[M]} \eta(f_i)
\end{equation}
where $\bar{f}$ is the deterministic ensemble classifier constructed via the rule:
\begin{equation}
    \bar{f}(\vc{x}) = \argmax_{c\in[C]} \left[\sum_{i=1}^M \tilde{f}_i(\vc{x})\right]_c
\end{equation}
Consider the following setup:%Fig~XX provides an illustration of this construction in $\reals^2$.  
\begin{enumerate}
    \item two binary classifiers in $\reals^2$:

\begin{equation}
    f_i(\vc{x}) = \begin{cases}
1 &\text{if $\tp{\vc{w}}_i\vc{x} \geq 0$}\\
2 &\text{otherwise}
\end{cases}
\end{equation}
which can be obtained from the \say{soft} classifiers:
\begin{equation}
    \tilde{f}_i(\vc{x}) = \begin{bmatrix}
   \tp{\vc{w}}_i\vc{x} \\
    -\tp{\vc{w}}_i\vc{x}
    \end{bmatrix}
\end{equation}
using $f_i(\vc{x}) = \argmax_{c\in \{1,2\}} [\tilde{f}_i(\vc{x})]_c$, where $\vc{w}_1 = \tp{[1\ 1]}$ and $\vc{w}_2 = \tp{[1\ -1]}$.
\item a $\mathsf{Ber}(p)$ data distribution $\calD$ over two data-points in $\reals^2 \times[2]$:
\begin{equation}
    \tuple{z}_1=\left(\vc{x}_1,y_1\right) = \left(\begin{bmatrix}
   -1\\
    2
    \end{bmatrix},1\right) \ \ \ \ \text{and} \ \ \ \     \tuple{z}_2=\left(\vc{x}_2,y_2\right) = \left(\begin{bmatrix}
   -1\\
    -2
    \end{bmatrix},1\right)
\end{equation} 

\item the $\ell_2$ norm-bounded perturbation set $\calS = \{\bm{\delta}: \norm{\bm{\delta}}\leq \epsilon\}$ for some $0<\epsilon<\nicefrac{1}{\sqrt{2}}$.
\end{enumerate}
We first note that for binary linear classifiers and $\ell_2$-norm bounded adversaries, we have that:
\begin{itemize}
    \item the shortest distance between a point $\vc{x}$ and the decision boundary of linear classifier $f$ with weight $\vc{w}$ and bias $b$ is:
    \begin{equation}
        \zeta = \frac{|\tp{\vc{w}}\vc{x} + b|}{\norm{\vc{w}}}
    \end{equation}
    
    \item if $f(\vc{x}) \neq y$, then the optimal adversarial perturbation is given  by:
    \begin{equation}
        \bm{\delta} = -\text{sign}\left(\tp{\vc{w}}\vc{x} + b\right) \frac{\epsilon\vc{w}}{\norm{\vc{w}}} 
    \end{equation}
\end{itemize}

We can now evaluate the adversarial risks of each classifier:
\begin{align}
    \begin{split}
        \eta_1 &= p \cdot\left(\max_{\norm{\bm{\delta}}\leq \epsilon} \identityf{\tp{\vc{w}}_1(\vc{x}_1+\bm{\delta}) <0}\right)  + (1-p)\cdot\left(\max_{\norm{\bm{\delta}}\leq \epsilon} \identityf{\tp{\vc{w}}_1(\vc{x}_2+\bm{\delta}) <0}\right) \\
        &= p \cdot\left( \identityf{1-\sqrt{2}\epsilon <0}\right)  + (1-p)\cdot\left( \identityf{-3<0}\right) = 1-p
    \end{split}
\end{align}
where we use $\epsilon <\nicefrac{1}{\sqrt{2}}$. Due to symmetry, we also get $\eta_2 = p$.

The average ensemble classifier $\bar{f}$ constructed from $f_1$ and $f_2$ is defined via the rule:
\begin{equation}
    \bar{f}(\vc{x}) = \begin{cases}
1 &\text{if $x_1 \geq 0$}\\
2 &\text{otherwise}
\end{cases}
\end{equation}
whose adversarial risk can be computed as follows:
\begin{align}
    \begin{split}
        \bar{\eta} &= p \cdot\left(\max_{\norm{\bm{\delta}}\leq \epsilon} \identityf{x_{1,1}+\delta_1 <0}\right)  + (1-p)\cdot\left(\max_{\norm{\bm{\delta}}\leq \epsilon} \identityf{x_{2,1}+\delta_1 <0}\right) \\
        &= p \cdot\left( \identityf{-1<0}\right)  + (1-p)\cdot\left( \identityf{-1<0}\right) =  p + 1-p = 1
    \end{split}
\end{align}

which is strictly greater than $\max\{p,1-p\}$  $\forall p\in (0,1)$. Therefore, we have constructed an example where deterministic ensembling is always worse than using any of the individual classifiers, which proves that deterministic ensemble classifiers \emph{do not} satisfy the upper bound.

\clearpage
\section{Additional Experiments and Comparisons}
\subsection{Experimental Setup} \label{app:setup}
In this section, we describe the complete experimental setup used for all our experiments.

\textbf{Training}. All models are trained for $100$ epochs via SGD with a batch size of $256$ and $0.1$ initial learning rate, decayed by $0.1$ first at the $50^{\text{th}}$ epoch and twice at the $75^{\text{th}}$ epoch. We employ the recently proposed margin-maximizing cross-entropy (MCE) loss from \cite{mrboost} with $0.9$ momentum and a weight decay factor of $5\times 10^{-4}$. We use $10$ attack iterations during training with $\epsilon=\nicefrac{8}{255}$ and a step size $\beta=\nicefrac{2}{255}$. For IAT, each classifier is indepdenelty trained from a different random initialization, using a standard PGD adversary. For MRBoost, we use their public implementation from GitHub to reproduce all their results. For BARRE, we use an adaptive PGD (APGD) adversary (discussed in detail in Section~\ref{app:arc-vs-apgd}) as our training attack algorithm. We apply OSP for $T_o=10$ iterations every $E_o=10$ epochs.

To avoid catastrophic overfitting \citep{rice2020overfitting}, we always save the best performing checkpoint during training. Since all the ensemble methods considered reduce to adversarial training for the first iteration, we use a shared adversarially trained first classifier. Doing so ensures a fair comparison between different ensemble methods. For both CIFAR-10, and CIFAR-100 datasets, we adopt standard data augmentation (random crops and flips). Per standard practice, we apply input normalization as part of the model, so that the adversary operates on physical images $\vc{x}\in[0,1]^d$.

\textbf{Evaluation}. For all our robust evaluations, we will adopt the state-of-the-art ARC algorithm \citep{dbouk2022adversarial} which can be used for both RECs and single models. Specifically, we use $20$ iterations of ARC, with an attack strength $\epsilon=\nicefrac{8}{255}$ and approximation parameter $G=2$. Following the recommendations of \cite{dbouk2022adversarial}, we use a step size of $\nicefrac{2}{255}$ when evaluating single models ($M=1$) and a step size of $\nicefrac{8}{255}$ when evaluating RECs ($M\geq 2$).

\subsection{Individual model robustness}

\begin{table}[tphb]
\centering
\caption{Natural and robust accuracies of the individual classifiers of all ensembles methods trained on CIFAR-10 (from Table~\ref{tab:barre-vs-others}). Robust accuracy is measured against an $\ell_\infty$ norm-bounded adversary using ARC with $\epsilon=\nicefrac{8}{255}$.} \label{tab:individual-cifar10}
\vskip 0.15in
%\begin{sc}
\resizebox{0.99\columnwidth}{!}{%
\begin{tabular}{l  l | r  r r r r r r r r r}

\toprule
\multirow{2}{*}{Network}  & \multirow{2}{*}{Method}  & \multicolumn{2}{c}{$f_1$}  & \multicolumn{2}{c}{$f_2$} & \multicolumn{2}{c}{$f_3$} &
\multicolumn{2}{c}{$f_4$} &
\multicolumn{2}{c}{$f_5$} 
\\
  &   & $A_{\normalfont\text{nat}}$ & $A_{\normalfont\text{rob}}$   & $A_{\normalfont\text{nat}}$  & $A_{\normalfont\text{rob}}$ & $A_{\normalfont\text{nat}}$ & $A_{\normalfont\text{rob}}$ & $A_{\normalfont\text{nat}}$ & $A_{\normalfont\text{rob}}$& $A_{\normalfont\text{nat}}$  & $A_{\normalfont\text{rob}}$  \\
  \midrule
  \multirow{3}{*}{ResNet-20} & IAT & \multirow{3}{*}{$73.18$} & \multirow{3}{*}{$41.99$} & $73.42$ & $41.94$ & $74.44$ &	$42.25$ & $74.27$ & $42.06$ & $74.17$ & $42.14$ \\
    & MRBoost &  &  & $76.00$ & $41.42$ & $76.59$ & $39.60$ & $77.25$ & $38.38$  & $76.43$ & $36.62$ \\
    & BARRE  &  &  & $76.08$ & $41.18$ & $77.40$ & $39.87$ & $77.12$ & $39.07$  & $77.60$ & $37.01$ \\
    \midrule
  \multirow{3}{*}{MobileNetV1} & IAT & \multirow{3}{*}{$79.01$} & \multirow{3}{*}{$46.22$} & $79.17$ & $46.21$ & $79.05$ &  $46.60$ & $78.44$ & $46.11$ & $78.76$ & $46.74$ \\
    & MRBoost &  &  & $80.11$ & $44.52$ & $77.54$ & $42.03$ & $77.94$ & $39.36$  & $68.89$ & $33.40$ \\
    & BARRE  &  &  & $80.15$ & $44.56$ & $79.43$ & $42.67$ & $79.56$ & $39.65$  & $79.60$ & $38.28$ \\
    \midrule
  \multirow{3}{*}{ResNet-18} & IAT &  \multirow{3}{*}{$80.96$} & \multirow{3}{*}{$48.72$} & $80.64$ & $48.23$ & $81.24$ & $48.83$ & $81.13$ & $48.70$ & $-$ & $-$ \\
    & MRBoost && & $84.01$ & $47.56$ & $83.67$ & $45.72$ & $83.88$ & $43.38$  & $-$ & $-$ \\
    & BARRE  && & $84.35$ & $46.48$ & $84.89$ & $45.86$ & $83.88$ & $43.09$  & $-$ & $-$ \\
\bottomrule

\end{tabular}
}
%\end{sc}
\end{table}
In Tables~\ref{tab:individual-cifar10}\&\ref{tab:individual-cifar100}, we provide the clean and robust accuracies of all the individual classifiers constructed via the different ensemble methods on CIFAR-10 and CIFAR-100, respectively. Robust accuracy is measured using ARC. 

As expected, only ensembles produced via IAT consist of classifiers achieving near-identical robust and natural accuracies. In contrast, ensembles produced via MRBosst or BARRE witness a degradation in individual classifier robust accuracy as the ensemble size grows. However, since MRBoost was not initially designed for randomized ensemble classifiers, this degradation in robust accuracy can be rather severe as seen for MobileNetV1 in both Tables~\ref{tab:individual-cifar10}\&\ref{tab:individual-cifar100}. This explains why, for such ensembles, the optimal sampling probability obtained for the constructed REC completely disregards the last classifier as highlighted in Section~\ref{ssec:experiments}.

\begin{table}[t]
\centering
\caption{Natural and robust accuracies of the individual classifiers of all ensembles methods trained on CIFAR-100 (from Table~\ref{tab:barre-vs-others}). Robust accuracy is measured against an $\ell_\infty$ norm-bounded adversary using ARC with $\epsilon=\nicefrac{8}{255}$.} \label{tab:individual-cifar100}
\vskip 0.15in
\begin{sc}
\resizebox{0.99\columnwidth}{!}{%
\begin{tabular}{l  l | r  r r r r r r r r r}

\toprule
\multirow{2}{*}{Network}  & \multirow{2}{*}{Method}  & \multicolumn{2}{c}{$f_1$}  & \multicolumn{2}{c}{$f_2$} & \multicolumn{2}{c}{$f_3$} &
\multicolumn{2}{c}{$f_4$} &
\multicolumn{2}{c}{$f_5$} 
\\
  &   & $A_{\normalfont\text{nat}}$ & $A_{\normalfont\text{rob}}$   & $A_{\normalfont\text{nat}}$  & $A_{\normalfont\text{rob}}$ & $A_{\normalfont\text{nat}}$ & $A_{\normalfont\text{rob}}$ & $A_{\normalfont\text{nat}}$ & $A_{\normalfont\text{rob}}$& $A_{\normalfont\text{nat}}$  & $A_{\normalfont\text{rob}}$  \\
  \midrule
  \multirow{3}{*}{ResNet-20} & IAT & \multirow{3}{*}{$38.34$} & \multirow{3}{*}{$17.69$} & $38.64$ & $17.68$ & $38.40$ & $17.89$ & $39.13$ & $17.63$  & $38.36$ & $18.13$ \\
    & MRBoost & & & $41.69$ & $17.29$ & $42.69$ & $17.67$ & $42.92$ & $17.44$  & $42.83$ & $16.11$ \\
    & BARRE  && & $41.57$ & $18.22$ & $42.96$ & $17.24$ & $42.69$ & $17.14$  & $43.72$ & $16.30$  \\
    \midrule
  \multirow{3}{*}{MobileNetV1} & IAT & \multirow{3}{*}{$51.87$} & \multirow{3}{*}{$23.45$} & $51.46$ & $23.01$ & $50.61$ & $23.00$ & $51.21$ & $23.40$  & $51.89$ & $23.56$ \\
    & MRBoost & & & $53.96$ & $22.63$ & $53.45$ & $20.48$ & $52.55$ & $19.90$  & $38.88$ & $11.34$ \\
    & BARRE  && & $52.75$ & $22.90$ & $53.61$ & $21.21$ & $54.31$ & $18.67$  & $51.99$ & $18.02$  \\
    \midrule
  \multirow{3}{*}{ResNet-18} & IAT & \multirow{3}{*}{$53.85$} & \multirow{3}{*}{$24.15$} & $53.85$ & $24.17$ & $54.80$ & $24.30$ & $54.71$ & $24.50$  & $-$ & $-$ \\
    & MRBoost & & & $54.78$ & $22.28$ & $47.49$ & $16.28$ & $48.13$ & $15.98$  & $-$ & $-$ \\
    & BARRE  && & $55.21$ & $22.26$ & $55.69$ & $21.05$ & $53.73$ & $17.99$  & $-$ & $-$  \\
\bottomrule

\end{tabular}
}
\end{sc}
\end{table}

\subsection{Attacks for Randomized Ensembles} \label{app:improved-arc}
Given a data-point $\tuple{z}=(\vc{x},y)$ and a potentially random classifier $f$, the goal of an adversary is to find an adversarial perturbation that maximizes the single-point expected adversarial risk:

\begin{equation}\label{eq:attack-r}
	\bm{\delta}^* = \argmax_{\bm{\delta}:\pnorm{\bm{\delta}}{p} \leq \epsilon} r(\tuple{z}, \bm{\delta}) =  \argmax_{\bm{\delta}:\pnorm{\bm{\delta}}{p} \leq \epsilon} \means{f}{\identityf{f(\vc{x}+\bm{\delta})\neq y}} = \argmax_{\bm{\delta}:\pnorm{\bm{\delta}}{p} \leq \epsilon} \prob{f(\vc{x}+\bm{\delta})\neq y}
\end{equation}
where we adopt the $\ell_p$ norm-bounded adversary for the remainder of this section.
\begin{table}[b]
\centering
\caption{Comparing the success of different attack algorithms at fooling various RECs using $\ell_\infty$ norm-bounded attacks with $\epsilon=\nicefrac{8}{255}$ on CIFAR-10. For adaptive CW, we use an $\ell_2$ adversary with $\epsilon=2$. All the RECs are constructed with equiprobable sampling.} \label{tab:attacks-cifar10}
\vskip 0.15in
%\begin{sc}
%\resizebox{0.99\columnwidth}{!}{%
\begin{tabular}{l  l | r r r  r r r }

\toprule
Network & Method  & ACW ($\ell_2$) & AFGSM & APGD-L  & APGD-S & ARC &
ARC-R \\

  \midrule
  \multirow{3}{*}{ResNet-20} & IAT & $56.18$ & $57.48$& $49.31$ & $49.34$ & $46.73$ & \bfu{$45.77$}   \\
    & MRBoost-R  & $57.88$& $59.00$& $49.65$ & $49.61$ & $47.74$ & \bfu{$46.66$} \\
    & BARRE   & $58.96$ & $59.53$& $49.79$ & $49.75$ & $48.05$ & \bfu{$47.35$}  \\
    \midrule
  \multirow{3}{*}{MobileNetV1} & IAT & $59.55$ & $62.27$&  $52.94$ & $52.91$ & $50.68$ & \bfu{$49.57$}   \\
    & MRBoost-R & $56.06$ &$60.45$ &  $51.19$ & $51.02$ & $49.37$ & \bfu{$48.05$} \\
    & BARRE  & $59.50$ &$62.60$ & $52.16$ & $51.94$ & $51.16$ & \bfu{$49.91$}  \\
    \midrule
  \multirow{3}{*}{ResNet-18} & IAT & $59.57$&$64.37$ &  $54.50$ & $54.49$ & $52.42$ & \bfu{$51.43$}   \\
    & MRBoost-R & $59.86$&  $64.51$& $54.51$ & $54.23$ & $53.19$ & \bfu{$51.82$} \\
    & BARRE   & $60.23$&$66.39$ & $54.52$ & $54.07$ & $53.62$ & \bfu{$52.13$}  \\

\bottomrule

\end{tabular}
%}
%\end{sc}
\end{table}

Projected gradient descent (PGD) \citep{madry2018towards} is perhaps the most popular attack algorithm for solving \eqref{eq:attack-r} for the case of differentiable deterministic classifiers. Specifically, given a surrogate loss function $l$, such as the cross-entropy loss, PGD finds an adversarial $\bm{\delta}$ iteratively via the following:

\begin{equation} \label{eq:pgd}
    \bm{\delta}^{(k)} = \Pi_\epsilon^p \left( \bm{\delta}^{(k-1)} + \eta \mu_p\left(\nabla_{\vc{x}} l\left(\tilde{f}\left(\vc{x}+\bm{\delta}^{(k-1)}\right) ,y\right)\right)\right)
\end{equation}
where $\mu_p$ is the $\ell_p$ steepest direction projection operator, and $\Pi_{\epsilon}^p$ is the projection operator on the $\ell_p$ ball of radius $\epsilon$.

In order to adapt PGD for evaluating randomized ensemble classifiers, \cite{pinot2020randomization} first proposed adaptive PGD (APGD-L) using the expectation-over-transformation (EOT) method \citep{athalye2018obfuscated}, which uses \eqref{eq:pgd} with the expected loss function as follows:
\begin{equation} \label{eq:apgd-l}
    \bm{\delta}^{(k)} = \Pi_\epsilon^p \left( \bm{\delta}^{(k-1)} + \eta \mu_p\left(\nabla_{\vc{x}} \mean{l\left(\tilde{f}\left(\vc{x}+\bm{\delta}^{(k-1)}\right) ,y\right)}\right)\right)
\end{equation}
Note that the discrete nature of randomized ensembles allows for an exact computation of the expectation in \eqref{eq:apgd-l}.

Recently, \cite{mrboost} proposed a stronger version of adaptive PGD, where the expectation is taken at the softmax level (APGD-S). Using APGD-S, \cite{mrboost} were able to compromise the BAT defense. Independently, \cite{dbouk2022adversarial} studied the effectiveness of EOT-based adaptive attacks for evaluating the robustness of RECs, and concluded that such methods are fundamentally ill-suited for the task. Instead, they proposed the ARC attack (Algorithm 2 of \citep{dbouk2022adversarial}), which relied on iteratively updating the perturbation based on estimating the direction towards the decision boundary of each classifier and using an adaptive step size method. 

In this section, we propose a small modification to ARC (ARC-R) that proves to be quite more effective in the equiprobable setting. Specifically, instead of looping over the classifiers in a deterministic fashion based on the order of the sampling probability vector, we propose using a randomized order loop. This ensures that ARC is never biased towards certain classifiers. In fact, Table~\ref{tab:attacks-cifar10} demonstrates that ARC-R is better than EOT-adapted single classifier attacks (PGD, FGSM, and CW) and ARC \citep{dbouk2022adversarial} at evaluating the robustness of RECs on CIFAR-10, constructed with equiprobable sampling across various network architectures and ensemble training methods. Hence, we shall adopt this version of ARC for all our experiments.

\subsection{ARC vs. Adaptive PGD for BARRE}\label{app:arc-vs-apgd}

As highlighted in Section~\ref{ssec:experiments}, we find that ARC, despite being the strongest adversary, leads to poor performance when adopted as the training attack in BARRE. In this section, we investigate this phenomenon, as we study the performance of BARRE using two different attacks during training, APGD \citep{mrboost} and ARC \citep{dbouk2022adversarial}. Specifically, we train two RECs on CIFAR-10 using the ResNet-20 architecture. Both RECs share the same first classifier $f_1$, which is adversarially trained using standard PGD. The second classifier $f_2$ is trained via either APGD or ARC. 

\begin{figure}[t]
  \centering
    \includegraphics[width=0.8\columnwidth]{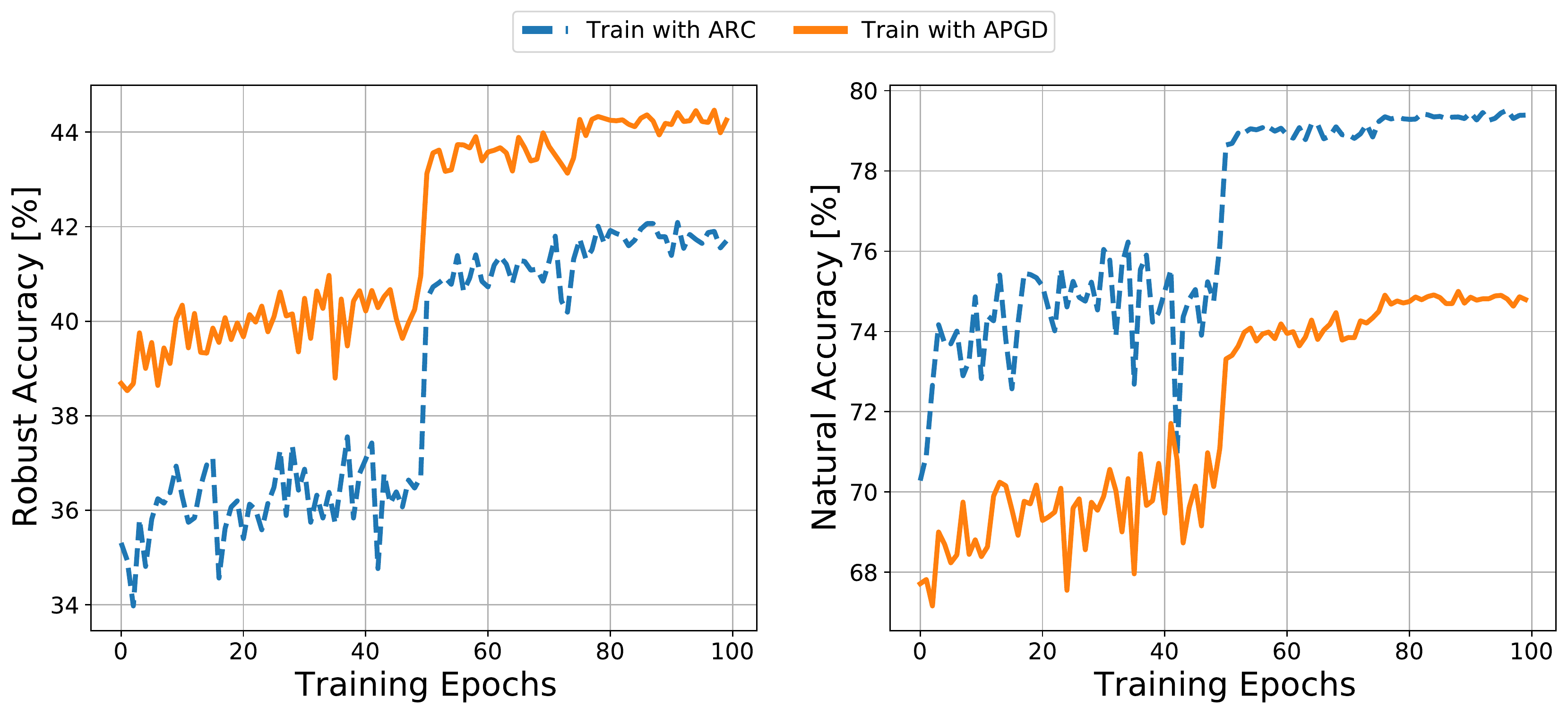}
    \caption{The robust (left) and natural (right) accuracies of an REC of two ResNet-20's trained on CIFAR-10 using BARRE vs. the training epochs of the second classifier $f_2$, where the first classifier $f_1$ is pre-adversarially trained. Robust and natural accuracies are reported on the test set, using $\ell_\infty$ norm-bounded adversaries with ARC and  $\epsilon=\nicefrac{8}{255}$.}
    \label{fig:train-arc-pgd}
\end{figure}

Figure~\ref{fig:train-arc-pgd} plots the evolution of both robust and clean accuracies of the two RECs across the 100 training epochs of $f_2$, measured on the test set. Note that in both RECs, the robust accuracy is evaluated via the stronger ARC adversary. When evaluated on clean images, we find that BARRE with ARC leads to significantly more accurate RECs when compared to BARRE with APGD. However, this comes at the expense of robust accuracy, as the REC obtained via BARRE with ARC is much more vulnerable than the APGD counterpart. We hypothesize that the adversarial samples generated via ARC during training do not generalize well to the test set. This explains why we observe that the REC obtained via BARRE with ARC achieves much higher robust accuracies on the \emph{training} set. Thus, for better generalization performance, we shall adopt adaptive PGD during training in all our experiments.

\subsection{Additional Results} \label{app:additional-results}
In this section, we complete the CIFAR-10 results reported in Table~\ref{tab:barre-vs-mrboost-cifar10} for showcasing the benefit of randomization. Specifically, Table~\ref{tab:barre-vs-mrboost-cifar100} provides further evidence that BARRE can train RECs of competitive robustness compared to MRBoost-trained deterministic ensembles, while requiring significantly less compute.

\begin{table}[tbhp]
\centering
\caption{Comparison between BARRE and MRBoost across different network architectures and ensemble sizes on CIFAR-100. Robust accuracy is measured against an $\ell_\infty$ norm-bounded adversary using ARC with $\epsilon=\nicefrac{8}{255}$.} \label{tab:barre-vs-mrboost-cifar100}
\vskip 0.15in

\resizebox{0.99\columnwidth}{!}{%
\begin{tabular}{l  l| r  r r | r  r r | r r r | r r r}

\toprule
\multirow{2}{*}{Network}  &  \multirow{2}{*}{Method}& \multicolumn{3}{c|}{$M=1$} & \multicolumn{3}{c|}{$M=2$}  & \multicolumn{3}{c|}{$M=3$} & \multicolumn{3}{c}{$M=4$}
\\
  &   & $A_{\normalfont\text{nat}}$ & $A_{\normalfont\text{rob}}$   & FLOPs & $A_{\normalfont\text{nat}}$ & $A_{\normalfont\text{rob}}$   & FLOPs & $A_{\normalfont\text{nat}}$  & $A_{\normalfont\text{rob}}$ &  FLOPs & $A_{\normalfont\text{nat}}$ & $A_{\normalfont\text{rob}}$ &  FLOPs \\
  \midrule
  \multirow{2}{*}{ResNet-20} & MRBoost & \multirow{2}{*}{$38.34$}& \multirow{2}{*}{$17.69$}& \multirow{2}{*}{81 M}&$41.08$ & $19.38$ & 162 M & $42.60$ & $20.48$ & 243 M & $43.62$ & $21.36$ & 324 M  \\ % & $77.41$ & $46.89$ & 405 M
    & BARRE &  &  &  &$39.95$ & $19.13$ &  81 M & $40.96$ & $19.85$ &  81 M& $41.40$ & $21.41$ &  81 M \\ %$76.28$ & $47.13$ & 81 M
  \cmidrule(lr{1em}){2-14}
  \multirow{2}{*}{MobileNetV1} & MRBoost & \multirow{2}{*}{$51.87$}& \multirow{2}{*}{$23.45$}& \multirow{2}{*}{312 M}& $54.41$ & $25.73$ & 624 M & $54.91$ & $26.63$ & 936 M & $55.03$ & $26.97$ & 1.2 B  \\ %$79.33$ & $49.48$ & 1.6 B 
    & BARRE &  &  &  & $52.31$ & $24.96$&  312 M & $52.74$ & $25.75$ &  312 M& $53.14$ & $27.12$ &  312 M\\
  \cmidrule(lr{1em}){2-14}
  \multirow{2}{*}{ResNet-18} & MRBoost & \multirow{2}{*}{$53.85$}& \multirow{2}{*}{$24.15$}& \multirow{2}{*}{1.1 B}& $55.83$ & $25.99$ & 2.2 B & $55.39$ & $26.09$ & 3.3 B & $55.80$ & $26.50$ & 4.4 B   \\ %& $84.67$ & $52.85$ & 5.5 B
    & BARRE &  &  &  & $54.53$ & $25.37$ &  1.1 B & $54.92$ & $25.76$ & 1.1 B& $54.63$ & $26.90$ &  1.1 B \\
\bottomrule

\end{tabular}
}

\end{table}

%%%%%%%%%%%%%%%%%%%%%%%%%%%%%%%%%%%%%%%%%%%%%%%%%%%%%%%%%%%%%%%%%%%%%%%%%%%%%%%
%%%%%%%%%%%%%%%%%%%%%%%%%%%%%%%%%%%%%%%%%%%%%%%%%%%%%%%%%%%%%%%%%%%%%%%%%%%%%%%

\end{document}